\documentclass[sigconf]{aamas} 
\usepackage{booktabs}
\usepackage{bbm}
\usepackage{mwe}
\usepackage{multirow}
\usepackage{subcaption}
\usepackage{balance} % for balancing columns on the final page
\usepackage{graphicx}
\graphicspath{{./figures/}}
\usepackage[ruled, lined, linesnumbered, commentsnumbered, longend]{algorithm2e}
\usepackage{xcolor}
\usepackage{color}
\SetCommentSty{mycommfont}
\newtheorem{prop}{\textbf{Proposition}}
\newtheorem{definition}{\textbf{Definition}}

\setcopyright{ifaamas}
\acmConference[AAMAS '23]{Proc.\@ of the 22nd International Conference
on Autonomous Agents and Multiagent Systems (AAMAS 2023)}{May 29 -- June 2, 2023}
{London, United Kingdom}{A.~Ricci, W.~Yeoh, N.~Agmon, B.~An (eds.)}
\copyrightyear{2023}
\acmYear{2023}
\acmDOI{}
\acmPrice{}
\acmISBN{}

\acmSubmissionID{1209}

\title{Model-based Dynamic Shielding for Safe and Efficient Multi-Agent Reinforcement Learning}

\author{Wenli Xiao}
\affiliation{
  \institution{The Chinese University of Hong Kong, Shenzhen}
  \institution{Shenzhen Institute of Artificial Intelligence and Robotics for Society}
  \country{China}
  }
\email{wenlixiao@link.cuhk.edu.cn}

\author{Yiwei Lyu}
\affiliation{
  \institution{Carnegie Mellon University}
  \country{United States}
  }
\email{yiweilyu@andrew.cmu.edu}

\author{John Dolan}
\affiliation{
  \institution{Carnegie Mellon University}
  \country{United States}
  }
\email{jdolan@andrew.cmu.edu}

\begin{abstract}
Multi-Agent Reinforcement Learning (MARL) discovers policies that maximize reward but do not have safety guarantees during the learning and deployment phases. Although shielding with Linear Temporal Logic (LTL) is a promising formal method to ensure safety in single-agent Reinforcement Learning (RL), it results in conservative behaviors when scaling to multi-agent scenarios. Additionally, it poses computational challenges for synthesizing shields in complex multi-agent environments. This work introduces Model-based Dynamic Shielding (MBDS) to support MARL algorithm design. Our algorithm synthesizes distributive shields, which are reactive systems running in parallel with each MARL agent, to monitor and rectify unsafe behaviors. The shields can dynamically split, merge, and recompute based on agents' states. This design enables efficient synthesis of shields to monitor agents in complex environments without coordination overheads. We also propose an algorithm to synthesize shields without prior knowledge of the dynamics model. The proposed algorithm obtains an approximate world model by interacting with the environment during the early stage of exploration, making our MBDS enjoy formal safety guarantees with high probability. We demonstrate in simulations that our framework can surpass existing baselines in terms of safety guarantees and learning performance.

\end{abstract}


\newcommand{\BibTeX}{\rm B\kern-.05em{\sc i\kern-.025em b}\kern-.08em\TeX}

\begin{document}

%%% The following commands remove the headers in your paper. For final 
%%% papers, these will be inserted during the pagination process.

\pagestyle{fancy}
\fancyhead{}

%%% The next command prints the information defined in the preamble.

\maketitle

\section{Introduction}
%Safety Challenge
Multi-Agent Reinforcement Learning (MARL)~\cite{MARL-survey-1, MARL-survey-2} is a promising approach to obtain learning control policies for multi-agent decision-making tasks such as transportation management~\cite{MARL-Application-transportation-1, MARL-Application-transportation-2}, motion control~\cite{MARL-Application-robot-1, MARL-Application-robot-2}, and autonomous driving~\cite{MARL-Application-car-1, MARL-Application-car-2, MARL-Application-car-3}. However, applying MARL methods in safety-critical autonomous systems (e.g., autonomous driving cars) can cause havoc due to the lack of formal safety guarantees. In addition, traditional MARL approaches with behavior penalties (i.e., giving a negative reward for unsafe actions) cannot ensure safety in practice~\cite{qin2021learning, MARL-shielding}. Therefore, there is a significant challenge to developing safe MARL systems that are provably trustworthy~\cite{MARL-survey-2, qin2021learning,lu2021decentralized,cai2021safe, MARL-shielding}.

\begin{figure}[t!]
    \centering
    \includegraphics[width=\linewidth]{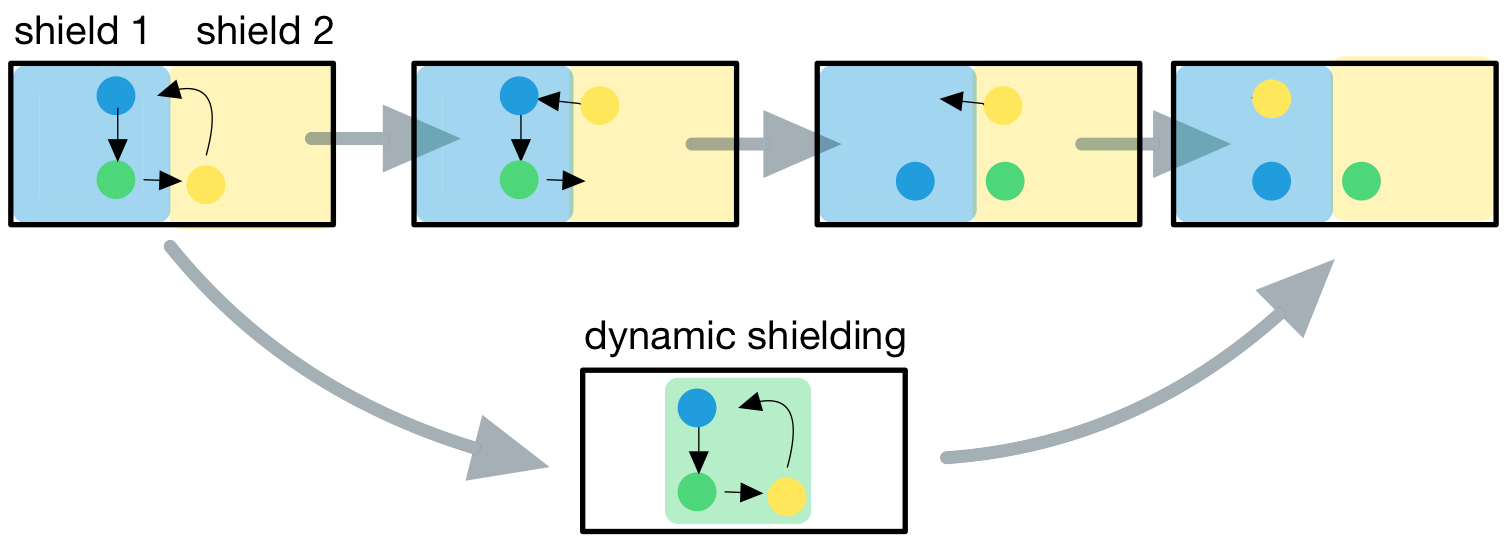}
    \caption{The differently colored circles denote multiple agents, and the black arrows are desired actions. Traditional decentralized shielding (upper) takes extra steps in waiting for coordination near the border of shields, while proposed dynamic shielding (lower) eliminates this overhead.}
    \label{fig:intro-conservative}
\end{figure}
%Shielding
Recently, there has been much research in notions of safety~\cite{MARL-safe-survey-1, safe-notation,RL-LTL-1,MARL-shielding,RL-shielding}. 
For example, Linear Temporal Logic (LTL)~\cite{LTL} is a specification language used for formal verification to ensure that an automation system always stays in safe states~\cite{rozier2011linear}. A recent work~\cite{RL-shielding} adopts LTL as a safety specification language in single-agent Reinforcement Learning (RL) via synthesizing a shield to monitor the RL agent. The shield is a lightweight system running along with the RL agent, which monitors actions selected by the RL agent and rejects any unsafe actions according to the given safety specification. The shield has provable safety guarantees for the lifetime of the RL process (i.e., the training and deployment phases). Factored shielding~\cite{MARL-shielding} adapts the shielded learning method to multi-agent scenarios in a decentralized fashion. Compared with centralized shielding, which uses one shield to monitor the states and actions of all agents, factored shielding synthesizes multiple shields, and each shield monitors a subset of the agents' state space. These methods perform well in discrete environments. However, both centralized shielding and factored shielding are challenging to scale up for more complex continuous environments.

%Delimma
On the other hand, when it comes to shielding framework design, there is a dilemma: centralized approaches have limited scalability~\cite{MARL-shielding}, while fully decentralized methods cause coordination overheads. Agents can become stuck waiting for coordination when they get closer to one another due to the lack of information sharing in decentralized approaches. For instance, Figure \ref{fig:intro-conservative} shows a scenario in which factored shielding causes extra coordination overhead. In this paper, we propose a novel, safe, and efficient MARL framework in a mixed decentralized manner, which dynamically synthesizes shields to mitigate those limitations. %without incurring sub-optimal policies.

%Contribution
Specifically, our main contributions are threefold: Firstly, we propose a novel shield framework - \textit{dynamic shielding}, which enables robots to collaborate to ensure safety. There are initially multiple shields, which concurrently monitor different agents. When there is a high risk of conservative behavior (e.g., agents move together), the shields could choose to merge with others. The merged shield can leverage the state information of multiple agents to mitigate unnecessary coordination overhead. When agents move apart from each other, the merged shield can split into multiple shields. We also present an effective shield synthesis approach in section~\ref{sec:ssynthesis}, named \textit{k-step look ahead shields}. Our method prunes the unnecessary computation of traditional shield synthesis approaches~\cite{RL-shielding, MARL-shielding} and delegates the computation complexity to the online algorithm, which can synthesize shields in real-time. 
We also incorporate a world-model learning procedure to learn a simplified environment dynamics model. This enables our framework to learn from scratch, with minimal external knowledge.

%Experiments
Additionally, we showcase the effectiveness and performance of our shielding approach through extensive experiments. We study the navigation problem on six different grid world maps~\cite{MARL-env-1} and two different tasks in the Multi-Agent Particle Environment~\cite{mpe} (MPE). Our approach outperforms other baselines in terms of reward and minimal steps while guaranteeing safety. Furthermore, we show that dynamic shielding ensures safety with a high probability as the number of agents scales up.

\section{Related Work}

\subsection{Safe Multi-Agent Reinforcement Learning}
Safe RL methods can be classified into two categories~\cite{MARL-safe-survey-1}: 1) The first is optimization criterion-based methods, which modify the RL objective functions~\cite{wachi2020safe, chow2018lyapunov, thananjeyan2021recovery}. For example, SNO-MDP~\cite{cmdp-1} tackles the safe RL problem using a constrained Markov decision process. 2) The second is based on modifying the exploration process to avoid undesirable actions~\cite{bastani2021safe, li2020robust, jansen2020safe}, which incorporates extra domain-specific knowledge (e.g., and demonstration) into the training process. Our dynamic shielding algorithm falls into the second category. Shielding was introduced to RL in~\cite{RL-shielding}, and was adapted to multi-agent settings in~\cite{MARL-shielding}. In this work, we propose a novel shielding framework for MARL by addressing challenges such as coordination overhead and scalability issues in the multi-agent setting and mitigating the reliance on external knowledge.

\subsection{Safe Control via Control Barrier Functions}
Barrier certificates~\cite{prajna2007framework,luo2020multi} and Control Barrier Function (CBF)~\cite{wieland2007constructive} based control methods~\cite{chen2017obstacle,cheng2019end,taylor2020learning,lyu2021probabilistic,lyu2022adaptive,van2022provable} are commonly used to provide safety guarantee for safety-critical problems, such as collision avoidance~\cite{zeng2021safety,wieland2007constructive,prajna2007framework,lyu2022responsibility}. 
In the context of multi-agent collision avoidance, previous research has explored using the multi-agent CBF frameworks~\cite{borrmann2015control,wang2017safety,lyu2022responsibility}. Recent works ~\cite{chen2020guaranteed,qin2021learning} have proposed decentralized controller synthesis approaches under the CBF that can scale to an arbitrary number of agents. We acknowledge their contributions, but they are perpendicular to our focus.

In this work, we aim to address more general safety specifications for MARL by leveraging a more expressive Linear Temporal Logic (LTL)~\cite{LTL}. LTL can conveniently capture complex time-varying constraints~\cite{srinivasan2020control}. Although we conduct experiments for collision-avoidance tasks in section~\ref{sec:exp}, we focus on the use of LTL for MARL in this work, with the aim of extending our approach to even more complex safety constraints in the future.

\subsection{LTL as Safety Specification} 
LTL is a widely used specification language in safety-critical systems~\cite{LTL-app-1, LTL-app-2}, which can express complex requests at a high level. For example, LTL has been used to express complex task specifications for robotic planning and control~\cite{kress2009temporal,ulusoy2013optimality}. 
Several works~\cite{bozkurt2020control, hahn2019omega, hasanbeig2020cautious} develop reward shaping techniques that translate logical constraints expressed in LTL to reward functions for RL. However, ~\cite{MARL-shielding,qin2021learning} has empirically demonstrated reward shaping cannot ensure safety in MARL.
Our shield synthesis technique, which is based on solving two-player safety games, was originally developed in~\cite{konighofer2017shield} to enforce LTL specifications. The original technique synthesizes shields to local caches in an offline manner. In section~\ref{sec:ssynthesis}, we propose a novel online method to synthesize shields in real time.
\section{Preliminaries}

We start by introducing \textit{Multi-Agent Reinforcement Learning}, \textit{Shielding}, and \textit{Safety Games with Linear Temporal Logic specification}, upon which our algorithm builds.

\subsection{Multi-Agent Reinforcement Learning} 
We focus on the $n$-player Markov Games defined by a tuple $$\left(\mathcal{N}, \mathcal{S},\left\{\mathcal{A}^{i}\right\}_{i \in \mathcal{N}},\left\{r^{i}\right\}_{i \in \mathcal{N}}, \mathcal{P}, \mathcal{\gamma}\right)$$ where $\mathcal{N}=\{1...n\}$ is the set of $n$ agents, $\mathcal{S}$ denotes the state space jointly observed by all agents, $\mathcal{A}^{i}$ is the action space of agent $i$, $r^i$ is the reward function of agent $i$, $\mathcal{P}: \mathcal{S} \times \mathcal{A} \rightarrow \Delta(\mathcal{S})$ denotes the transition probability, and $\mathcal{\gamma}$ is the discount factor. We assume the initial state $s_1$ follows a fixed distribution $\rho \in \Delta(\mathcal{S})$. At each time step $t$, the agents observe state $s_t$, take actions $a_{t,i} \in A^i$ in the environment simultaneously, and receive rewards $r_{t,i} \in R^i$. Then the state of the environment moves to $s_{t+1}$. The objective of each agent $i$ is to learn a control policy $\pi_i$ which maximizes the expected cumulative reward $E\left[\sum_{t=0}^{\infty} \gamma^{t} R^{i}\left(s_{t}, a_{t}, s_{t+1}\right)\right]$. MARL algorithms can be categorized into three different types based on the dependence of individual agent performance on other agents' choices, including cooperative, competitive, and mixed settings. We use MARL algorithms with mixed settings in our experiment in Section~\ref{sec:exp}. CQ-learning~\cite{cqlearning} is a MARL algorithm that enables agents to behave separately at most of the time and consider the states and actions of other agents when necessary. MADDPG~\cite{mpe} is a deep MARL algorithm with centralized training and decentralized execution, each agent trains models that simulate each of the other agents' policies based on its observation of their behavior.

\subsection{LTL as Safety Specification}

We consider Linear Temporal Logic~\cite{LTL} (LTL) to express safety specifications. LTL is an extension of propositional logic, which has long been used as a tool in the formal verification of programs and systems. The syntax of LTL is given by the following grammar~\cite{LTL-syntax}:
$$
\varphi:=p|\neg p| \varphi_{1} \vee \varphi_{2}|\bigcirc \varphi| \varphi_{1} \mathcal{U} \varphi_{2}
$$
where $p$ is an atomic proposition. The temporal operators are next $\bigcirc \varphi$, which indicates $\varphi$ is true in the next succeeding state, and until $\varphi_{1} \mathcal{U} \varphi_{2}$ indicating $\varphi_1$ is true until the state where $\varphi_2$ is true. From these operators, we can define $True \equiv \phi  \vee \neg \phi$, $False \equiv \neg True$, implication $\varphi \Rightarrow \psi:=\neg \varphi \vee \psi$, eventually $\diamond \varphi:=\operatorname{True} \mathcal{U} \varphi$, and always  $\square \varphi:=\neg \diamond \neg \varphi$. We use LTL formulas to express safe specifications. For example, \textit{$\square \neg collision$} denotes that collision should never happen. We consider translating the LTL safety specification into a safe language accepted by a deterministic finite automaton (DFA)~\cite{dfa}. In addition, we extend the definition of safe RL in \cite{RL-shielding} to MARL in the following way:
\begin{definition}
    \label{def:safety}
  Safe MARL is the process of learning optimal policies for multiple agents while satisfying a temporal logic safety specification $\phi^s$ during the learning and execution phases.
\end{definition}

\subsection{Formal Safety Guarantee with Shield}

\begin{figure}[ht!]
    \centering
    \includegraphics[width=0.5\linewidth]{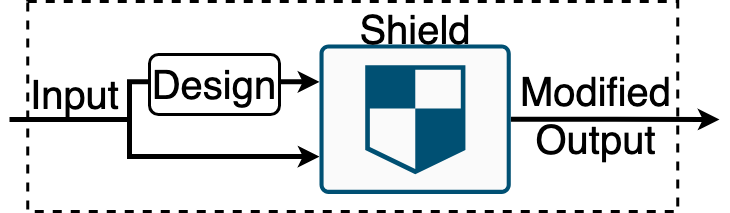}
    \caption{Enforce safety specification via shielding.}
    \label{fig:intro-shield}
\end{figure}

Our method builds upon a prior method called Shield~\cite{MARL-shielding, konighofer2017shield}, which ensures safety properties at runtime. A Shield (shown in Figure~\ref{fig:intro-shield}) monitors the control input of agents and corrects any unsafe control input instantaneously. A Shield should have two properties: 1) Minimal interference. Namely, shields only correct the action if it violates the safety rule. 2) Correctness. Shields should distinguish every unsafe action and refine it with safe actions. Our method uses the Shield framework to ensure safety, and we provide theoretical proof of safety in section~\ref{sec:ssynthesis}. 

We \textbf{represent the shield} using a finite-state reactive system. According to the formulation in ~\cite{MARL-shielding}, a finite-state reactive system is a tuple $S=\left(Q, q_{0}, \Sigma_{I}, \Sigma_{O}, \delta, \lambda\right)$, where $\Sigma_{I}$ and $\Sigma_{O}$ are the I/O alphabets, $Q$ is the state set, $q_0 \in Q$ denotes the initial state, $\delta : Q \times \Sigma_{I} \rightarrow Q$ is a transition function, and $\lambda : Q \times \Sigma_{I} \rightarrow \Sigma_O$ is an output function. Given the symbolic abstraction of the control input (i.e., input trace) $\overline{\sigma_{I}}=x_{0} x_{1} \ldots \in \Sigma_{I}^{\infty}$, the system S generates the trajectory of states (i.e., output trace) $\overline{\sigma_{O}}=\mathcal{S}\left(\overline{\sigma_{I}}\right)= \lambda\left(q_{0}, x_{0}\right) \lambda\left(q_{1}, x_{1}\right) \ldots \in \Sigma_{O}^{\infty}$, where $q_{i+1} = \delta(q_i, x_i)$ for all $i\geq 0$.

We \textbf{synthesize the shield} by solving a \textit{two-player safety game}~\cite{konighofer2017shield}, a game played by the MARL agents and the environment, where the winning condition is defined by the LTL safety specification. MARL agents should comply with all safety specifications all of the time in order to win the game. A two-layer game is a tuple $\mathcal{G}=\left(G, g_{0}, \Sigma_{I}, \Sigma_{O}, \delta,win\right)$ with a finite set of game states $G$, the initial state $g_0 \in G$, a complete transition function $\delta : G \times \Sigma_I \times \Sigma_O \rightarrow G$, and $win$ as a winning condition. In every state $g\in G$, the environment first chooses an input action $\sigma_I \in \Sigma_I$, and then the MARL agents choose a joint action (in abstraction symbol) $\sigma_O \in \Sigma_O$. Then the game moves to the next state $g' =\delta(g,\sigma_I,\sigma_O)$, and so forth. The resulting trajectory of game states $\bar{g} = g_0, g_1, ...$ is called a \textit{play}. A play is won if and only if $win(\bar{g})$ is $true$. We describe the detailed procedure of synthesizing shields via solving the two-player safety game in section~\ref{sec:ssynthesis}.

\begin{figure}[htb]
    \centering
    \includegraphics[width=0.7\linewidth]{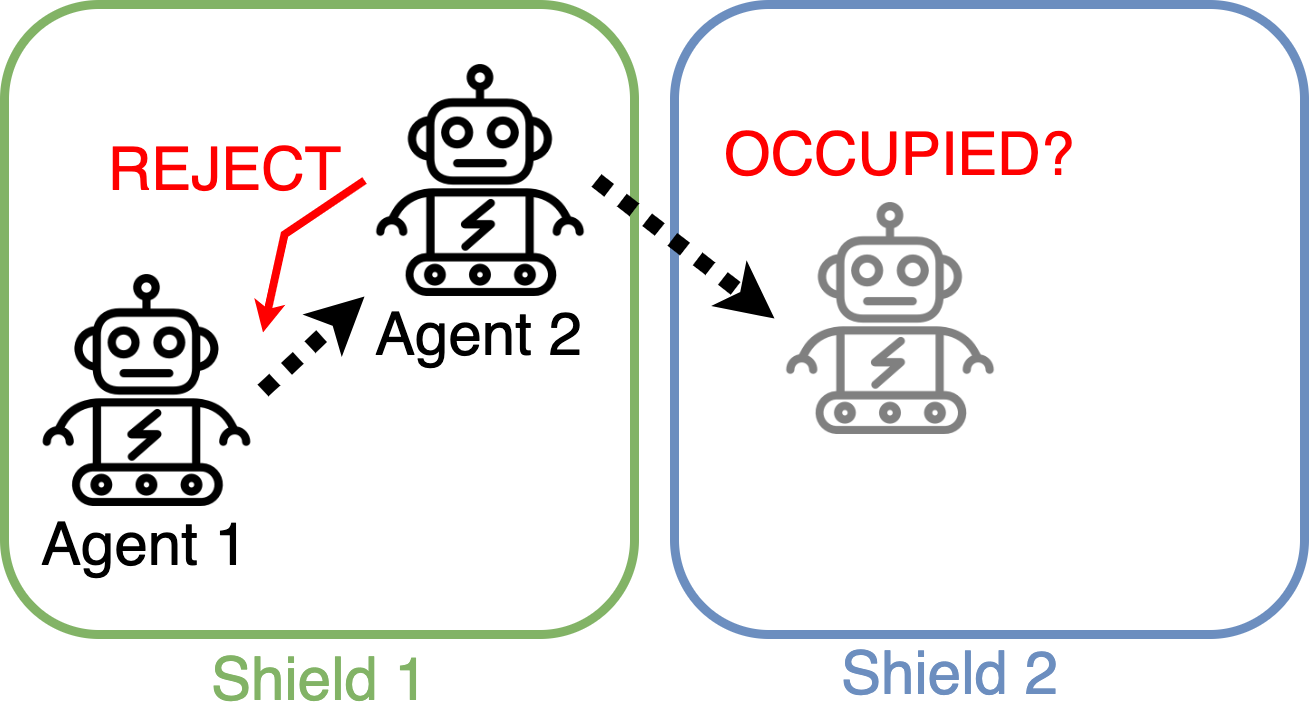}
    \caption{The green and blue squares denote shields, and the dashed arrows are desired actions of agents. There is no communication between shields 1 and 2. Therefore, Shield 1 conservatively judges that agent 2 cannot successfully enter shield 2, thus rejects Agent 1's action.}
    \label{fig:example-conservative}
\end{figure}

\begin{figure*}[htb]
    \centering
    \begin{subfigure}[t]{0.4\textwidth}
    \centering
    \includegraphics[width=\textwidth]{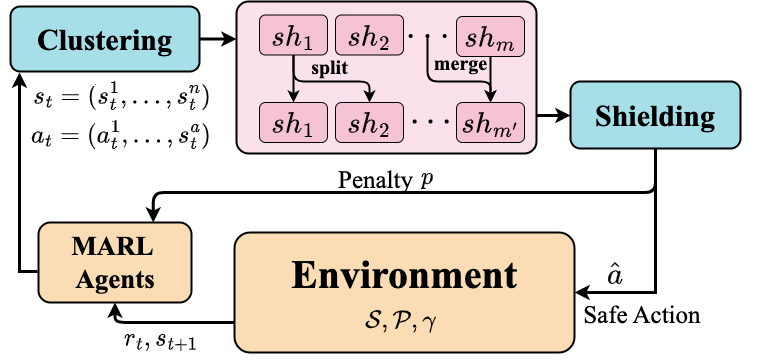}
    \caption{Learn control policies with the protection of Dynamic Shielding.}
    \label{fig:frame-dynamic-shielding}
    \end{subfigure}
    ~
    \begin{subfigure}[t]{0.4\textwidth}
    \centering
    \includegraphics[width=\textwidth]{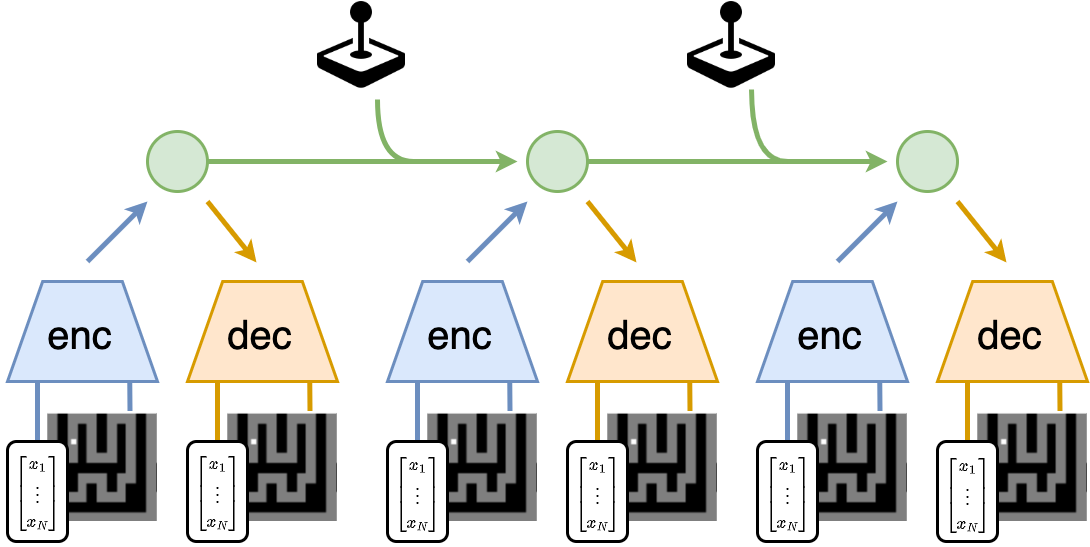}
    \caption{Learn Dynamics from experience.}
    \label{fig:frame-world-model}
    \end{subfigure}
    \caption{Dynamic Shielding Framework}
\end{figure*}

\section{Tackling Safe and Efficient Multi-agent Reinforcement Learning via Dynamic shielding}
\label{sec:dshield}

In this section, we first describe how traditional shielding methods cause sub-optimal learning. Then, we present our method for safe, optimal, and optimal MARL learning.

\subsection{Conservative Behavior and Coordination Overhead}
For multi-agent systems, centralized approaches always fail when the number of agents increases. For example, centralized shielding for MARL fails empirically for two-agent scenarios~\cite{MARL-shielding}. Fully decentralized shielding separates the whole state space into exclusive subspaces and synthesizes a shield to monitor a subspace. For example, factored shielding~\cite{MARL-shielding} computes multiple shields based on a factorization of the joint state space observed by all agents. However, this approach causes conservative behaviors (i.e., agents stuck in place) when agents move across the border of shields due to the information isolation between shields. Specifically, as shown in Figure~\ref{fig:example-conservative}, the shield would reject agents' actions even for those that are essentially valid. Consequently, the MARL system has higher coordination overhead, which causes extra steps when agents interact and render the MARL policy sub-optimal. In Section~\ref{sec:exp}, we empirically demonstrate that the coordination overhead caused by conservative behaviors leads to sub-optimal policies.

\subsection{Dynamic shielding}

\begin{algorithm}
    \SetKwFunction{isOddNumber}{isOddNumber}
    \SetKwFunction{reCompute}{reCompute}
    % \SetKwFunction{Or}{Or}
    \SetKwInput{Input}{Input}
    \SetKwInput{Output}{Output}
    \SetKwInOut{KwIn}{Initialize}
    \SetKwInOut{KwOut}{Output}

    \KwIn{A list of shields $S = \{s_1, s_2, ..., s_m\}$, MARL agents' joint action $a_t = (a_t^1,a_t^2, ... a_t^n)$ and joint state $s_t = (s_t^1, s_t^2, ..., s_t^n)$, a constant penalty for unsafe actions $p$, a environment dynamics model $p(s_t | s_{t-1}, a_{t-1} )$}
    \KwOut{Safe joint action $\bar{a}_t$, punishment $p_t$, shield new\_shield}

    \textcolor{purple}{// Clustering: divide agents into groups} \\
    \textcolor{purple}{//  e.g., cluster agents by their position} \\
    
    $new\_shield = cluster\_agents(s_t, a_t)$

    \For{all group in $i \in \{1, ..., m'\}$}{
        \For{all shield $j \in \{, ..., m\}$}{
            \If{new\_shield[I].group == $s_j.group$ \\and $s_j$.duration $!=$ 0}{
                 new\_shield[i].recompute = False \\
                 new\_shield[i].shield = $s_j.shield$\\
            }
        }
    }
    \textcolor{purple}{// Re-construct shields} \\
    \textcolor{purple}{// 1. When agents trying to escape shields} \\
    \textcolor{purple}{// 2. When shields expire} \\
    
    \For{all group in $i \in \{1, ..., m\}$}{
        \If{$new\_shield[i].recompute == True $}{
            new\_shield[i].$\reCompute(p)$\\
        }
    }

    \textcolor{purple}{// Shielding: Replace unsafe actions to safe actions}\\
    $\bar{a}_t$ = safe action output by new\_shield 
    
    \For{all agent in $i \in \{1, ..., n\}$}{
        \If{$\bar{a}_t^i \neq a_t^i$}{
            $p_t^i$ = p\\
        }
    }
    \KwRet{$bar{a}_t$, $p_t$, new\_shield}
    \caption{Dynamic Shielding at timestep $t$}
    \label{algo:dshield}
\end{algorithm}

%Merge, Split, Algorithm Block, Figure
To mitigate the coordination overhead caused by conservative behaviors, we propose \textit{dynamic shielding}, a decentralized shield framework on top of the traditional MARL process. Dynamic shielding has two important features: 1) Dynamic shielding dynamically constructs shields based on agents' real-time states; 2) Dynamic shielding can perform two important operations, namely, \textit{merge} and \textit{split}. The merge operation uses multiple shields' information to construct a larger shield, which temporarily removes the border between shields. Therefore, the merged shield has enough information to distinguish whether joint actions are safe and eventually mitigate conservative behavior. On the other hand, the computation complexity in shield synthesis increases along with the shield size. The split operation helps decrease computation costs when agents locate sparsely. Figure~\ref{fig:frame-dynamic-shielding} shows the diagram of dynamic shield construction. Initially, we construct distinct shields for each agent, which monitor agents' reachable states in the next $k$ steps. If agents try to move to states outside the shield, the shield will recompute to establish a monitor on agents' future possible states. When agents gathering together has the possibility of collision, shields will merge to jointly monitor the action using the state information of multiple agents. When agents are more sparse, the merged shield will split to save computation. 

We summarize dynamic shielding in algorithm~\ref{algo:dshield}. There are three phases: 1) clustering, 2) shield reconstruction, and 3) shielding. In the clustering phase (LINE 1-12), the algorithm clusters agents into groups by their current state. For example, in robot navigation tasks, if some agents are close by, the algorithm will put them in the same group, otherwise in separate groups. Agents in the same group should merge shields to avoid conservative behaviors. Then, in the shield re-construction phase (LINE 13-20), shields will merge with other shields or split into multiple smaller shields based on the results of clustering. In addition, some expired shields might recompute according to agents' state change. The merge operation could be implemented by recompute shield using agents' aggregated state information. In the shielding phase (Lines 21-27), every shield will do shielding concurrently, which rejects agents' unsafe actions and replaces them with safety actions. Lastly, the MARL agents will be given an extra penalty for unsafe actions.

Our method faces the challenge it degrades to centralized shielding for edge scenarios (some timesteps). For example, when all agents gather together, all decentralized shields will merge together into a single centralized shield. We propose an online method of shield synthesis in Section~\ref{sec:ssynthesis}, which could efficiently synthesize shields.

\begin{algorithm}
    % \SetKwFunction{isOddNumber}{isOddNumber}
    % \SetKwFunction{reCompute}{reCompute}
    % \SetKwFunction{Or}{Or}
    % \SetKwInput{Input}{Input}
    % \SetKwInput{Output}{Output}
    \SetKwInOut{KwIn}{Initialize}
    % \SetKwInOut{KwOut}{Output}

    \KwIn{Initialize dataset $\mathcal{D}$.}
    \KwIn{Initialize neural network parameters $\theta$ randomly.}

    \While{\text{not converged}}{
        \textcolor{purple}{// Dynamics Learning} \\
        \For{\text{update step} $c=1...C$}{
            Draw B data sequences $\{\left(o_t,a_t\right)\}^{k+L}_{t=k}$\\
            Update $\theta$. \textcolor{purple}{// $p_\theta\left(o_t\mid o_{t-1},a_t\right)$} \\
        }
        \textcolor{purple}{// Collecting Data} \\
        $o_1 = \operatorname{env.reset}() $ \\
        \While{\text{episode not stopped}}{
            $a_t \sim A$ \\
            $o_{t+1} \leftarrow \operatorname{env.step}\left(o_t,a_t\right)$
        }
        \text{Add experience to dataset $\mathcal{D}$}
    }
    \KwRet{$\theta$}
    \caption{Learn dynamics model}
    \label{algo:world_model}
\end{algorithm}

\section{Synthesize shield in real-time}
\label{sec:ssynthesis}

In this section, we introduce the incorporation of world model learning and present our shield synthesis method -- \textit{k-step look ahead shields}, a variant of traditional shield synthesis~\cite{konighofer2017shield}. We also give theoretical proof to show that our method guarantees safety.

\subsection{Learn the environment dynamics}
To cope with the scenario that the MARL agents do not have external knowledge about the environment in the first place, our framework will train a coarse world model at the beginning. This world model is a deep neural network that learns to predict the environmental dynamics related to safety considerations. For example, an autonomous driving car could have different sensor inputs such as Lidar, Cameras, and GPS. If we only care about safety related to locomotion, the coarse world model will be trained to predict GPS signals under control inputs. There are many existing works on world model learning~\cite{RSSM,ha2018world,hafner2019learning}. We adapt a general framework from Recurrent State-Space Model (RSSM)~\cite{RSSM}, which consists of three components (Fig~\ref{fig:frame-world-model}): encoder, decoder, and dynamics networks, which are denoted by $\quad \operatorname{enc}_\theta\left(s_t \mid s_{t-1}, x_t\right)$, $\quad \operatorname{dec}_\theta\left(s_t\right) \approx x_t$, and $\quad \operatorname{dyn}_\theta\left(s_t \mid s_{t-1}, a_{t-1}\right)$ respectively.

\iffalse
\begin{align*}
    \text{Encoder Network:} & \quad \operatorname{enc}_\theta\left(s_t \mid s_{t-1}, x_t\right)  \\
    \text{Decoder Network:} & \quad \operatorname{dec}_\theta\left(s_t\right) \approx x_t \\
    \text{Dynamics Network:} & \quad \operatorname{dyn}_\theta\left(s_t \mid s_{t-1}, a_{t-1}\right) 
\end{align*}
\fi

With such a model obtained by Algorithm~\ref{algo:world_model}, given state $x_t$ and action $a_t$ with latent of previous state $s_{t-1}$, we can predict the next state $x_{t+1}$ by:
\begin{align*}
    s_t &\sim \operatorname{enc}_\theta\left(s_t \mid s_{t-1}, x_t\right)\\
    s_{t+1} &\sim \operatorname{dyn}_\theta\left(s_{t+1} \mid s_{t}, a_{t}\right) \\
    x_{t+1} &\sim \operatorname{dec}_\theta\left(s_{t+1}\right)
\end{align*}

In this work, we focus on using the world model to learn the low-dimensional intrinsic properties of the environment, such as physical dynamics, for shield synthesis. We assume the existence of an expressive world model, which allows us to abstract away from the details of the sensory input and reason about the environment at a higher level of abstraction. For this work, the cascading error that may arise from errors in the learned world model is outside the scope of our discussion.

\subsection{$k$\-step lookahead shields}
We assume the state space has been converted into a symbolic abstraction (via $f: S \rightarrow L$) given by a DFA $\mathcal{A}^{e}=\left(Q^{e}, q_{0}^{e}, \Sigma^{e}, \delta^{e}, F^{e}\right)$. We translate the LTL safety specification into another DFA $\mathcal{A}^{S}=\left(Q^{S}, q_{0}^{S}, \Sigma^{S}, \delta^{S}, F^{S}\right)$. We formulate a two-player game $$\mathcal{G}=\left(G, g_{0}, \Sigma_{1}, \Sigma_{2}, \delta^{g}, F\right)$$ by combining $\mathcal{A}^{e}$ and $\mathcal{A}^{S}$. Instead of solving the game $G$ directly, we add extra time constraints $t\leq k$, where $t \in $ denotes the time step from constructing the shield, and $k$ is a hyper-parameter that denotes the maximum steps of the game. The modified game is then 
\begin{equation}
\mathcal{G}^k=\left(G^k, g_{0}{'}, \Sigma_{1}, \Sigma_{2}, \delta^{g}{'}, F^k\right)
\end{equation}
where the state space $G^k = G \times \{1 ... k\}$, the initial state $g_{0}{'} = (g_0, t=1)$, the transition function $\delta^{g}{'}(g_t, t) = (\delta^{g}(g_t), t+1)$, which could be approximated through the world model, and the winning condition $F^k = F \wedge (t \leq k)$. We can solve the two-player safety game $\mathcal{G}^k$ 
and compute the winning region $W \subseteq F^{k}$, using the method in~\cite{konighofer2017shield}. We then construct the \textit{k-step look ahead shield} by translating $\mathcal{G}^k$ and $W$ to a reactive system $S=\left(Q_{\mathcal{S}}, q_{0, \mathcal{S}}, \Sigma_{I, \mathcal{S}}, \Sigma_{O, \mathcal{S}}, \delta_{\mathcal{S}}, \lambda_{\mathcal{S}}\right)$. The shield has the following components: $Q_{\mathcal{S}}= G^k$, $q_{0, \mathcal{S}}=q_{0}{'}$, $\Sigma_{I, \mathcal{S}}=L \times \mathcal{A}$, $\Sigma_{O, \mathcal{S}}=\mathcal{A}$, $\delta_{\mathcal{S}}(g^k,(l, a))=\delta\left(g^k,\left(l, \lambda_{\mathcal{S}}(g,(l, a))\right)\right)$ for all $g^k \in G, l \in$ $L, a \in \mathcal{A}$, and
$$
\lambda_{\mathcal{S}}(g, l, a)= \begin{cases}a & \text { if } \delta^k(g^k,(l, a)) \in W \\ a^{\prime} & \text { if } \delta^k(g^k,(l, a)) \notin W \text { for some arbitrary } \\ & \text { default } a^{\prime} \text { with } \delta^k\left(g^k,\left(l, a^{\prime}\right)\right) \in W .\end{cases}
$$

Our shield synthesis bears a resemblance to the classic shield synthesis~\cite{MARL-shielding, RL-shielding}, which also synthesizes shields by solving the two-player game. The main difference is that our method only predicts a subset of future state space, whereas previous methods enumerate all possible states along the planning horizon. This leads to the major benefit of our method, that for tasks with state spaces too large to compute in advance, our algorithm still works efficiently while previous methods fail.

\subsection{Safety Guarantee}
We show that dynamic shielding with \textit{k-step look ahead shielding} can guarantee safety for MARL agents.

\begin{prop}
Given a trace $s_{0} a_{0} s_{1} a_{1} \cdots \in (S \times A)^\omega$ jointly produced by MARL agents, the dynamic shielding, and the environment, state-action pair $(s_t, a_t)$ is safe at every time step regarding definition~\ref{def:safety}. 
\end{prop}

\begin{proof}
Firstly, the procedure in algorithm~\ref{algo:dshield} ensures each agent is monitored by a shield at each time step, and this shield at least monitors the states of agents in the next $k$ steps (otherwise, the shield will re-compute). Then the remaining proof is the same as the correctness of centralized shielding in~\cite{MARL-shielding}. For any agents under shield $\mathcal{S}=\left(Q, q_{0}, \Sigma_{I}, \Sigma_{O}, \delta, \lambda\right)$, there is a corresponding run $q_0 q_1, ... q_m \in (S\times A)^\omega$, where $m \leq k$ is the duration before reconstructing this shield. By constructing the shield, we have the environment abstraction DFA $A^e$ and the safety specification DFA $A^s$. We can project the run $q_0,q_1, ...q_m$ of the shield $\mathcal{S}$ onto a trace $q_0^s(f(s_0),a_0)q_1^s(f(s_1),a_1) ...q_m^s(f(s_m),a_m) $ on $A_s$. Since we construct the shield from the winning region of the two-player safety game, every state $q_i^s(f(s_i)$ visited by agents along the trace should be a safe state in $\mathcal{A}^s$. The shield $\mathcal{S}$ ensures the safety specification defined in $\mathcal{A}^s$ is never violated. Therefore, the joint state-action pair $(s_t,a_t)$ is safe for every MARL agent at every step.

\end{proof}

\begin{figure}[htp]
    \centering
    \includegraphics[width=0.65\linewidth]{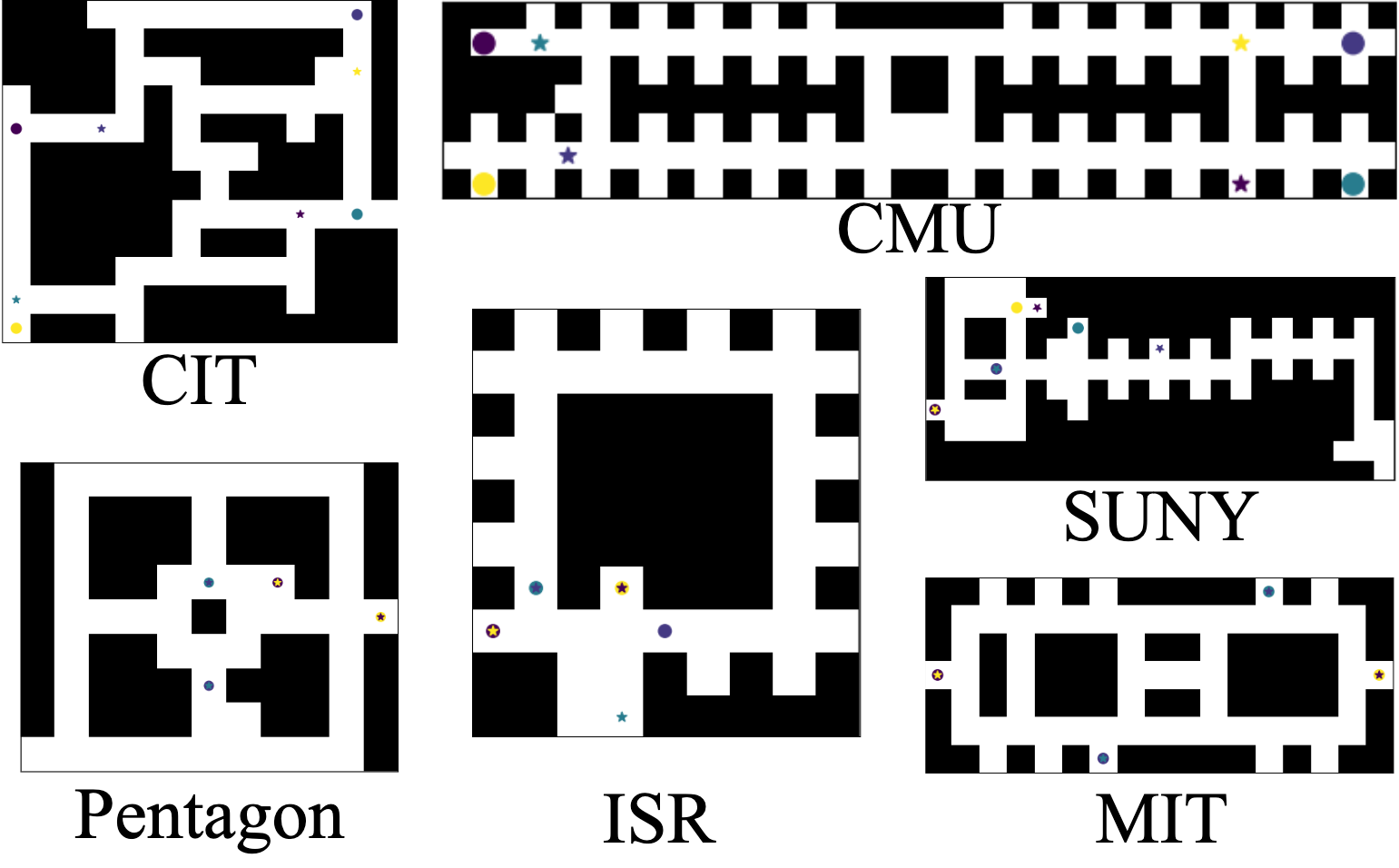}
    \caption{Different gridworld environments. Dots are agents, stars denote targets, and black blocks are obstacles.}
    \label{fig:exp_grid}
\end{figure}

\begin{figure}[htp]
    \centering

\begin{subfigure}[t]{0.5\linewidth}
    \centering
    \includegraphics[width=0.8\textwidth]{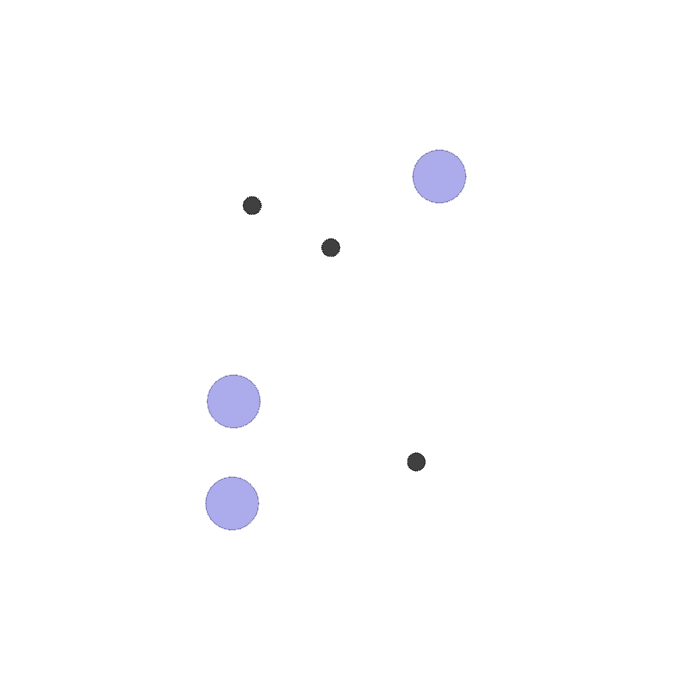}
    \caption{Simple Spread}
    % \label{fig:result_mpe_reward}
\end{subfigure}
~
\begin{subfigure}[t]{0.5\linewidth}
    \centering
    \includegraphics[width=0.8\textwidth]{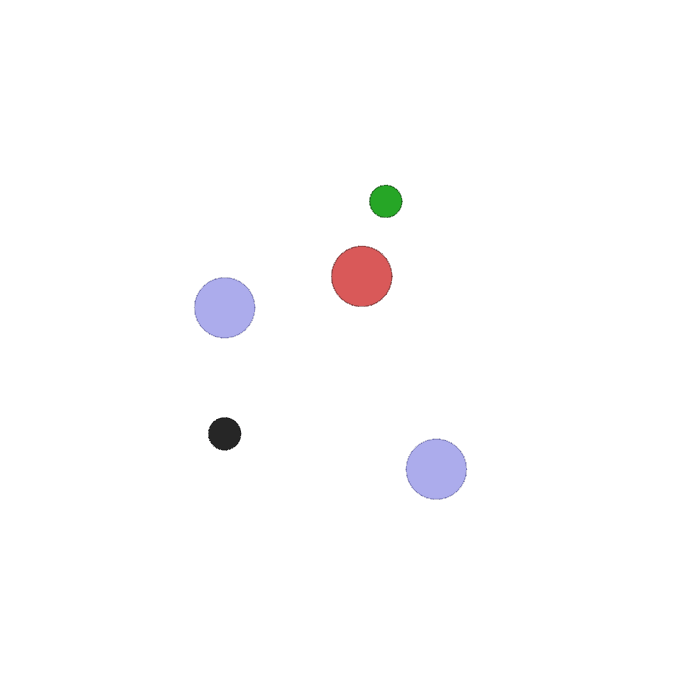}
    \caption{Simple Adversary}
    % \label{fig:result_mpe_reward}
\end{subfigure}
    
\caption{Modified Multi-Agent Particle Environment (MPE). The blue circles denote agents, black dots are landmarks. In \textit{Simple Adversary}, the red circle is an adversary agent, and the green dot is the target landmark, we let $n_{good}:n_{adversary}=3:1$. The MPE environment is unbounded, but agents will be penalized if they move too far away.}
\label{fig:exp_mpe}
    
\end{figure}

\section{Experiments}

\label{sec:exp}
In this section, we empirically evaluate the performance of the proposed safe MARL framework (Algorithm~\ref{algo:dshield}) on multiple benchmark tasks. We apply our algorithm to four different maps of the \textit{gridworld}~\cite{gridworld} (shown in Figure~\ref{fig:exp_grid}) and two environments (cooperative and mixed-cooperative of MPE (shown in Figure~\ref{fig:exp_mpe}). We compare the proposed algorithm with CQ-Learning~\cite{cqlearning}, CQ with factored shielding (CQ+FS), DDPG~\cite{ddpg}, MADDPG~\cite{mpe}, and MADDPG+CBF~\cite{cai2021safe}. We conducted three sets of experiments in total. For the first two sets, we assume known environment dynamics to evaluate the performance of our shielding framework conveniently. In the last experiment, we train our framework without any external knowledge of the environment to demonstrate the effectiveness of MARL with the proposed shielding in practice. 
We implement algorithms using Python and synthesize shields via Slugs~\cite{slugs}. For each experiment, we evaluate algorithms in both the training and testing phases. To mitigate outliers, we performed all experiments in 5 independent runs and averaged the results.

\begin{figure}[htp]
    \centering
    \begin{subfigure}[t]{\linewidth}
    \centering
    \includegraphics[width=0.8\textwidth]{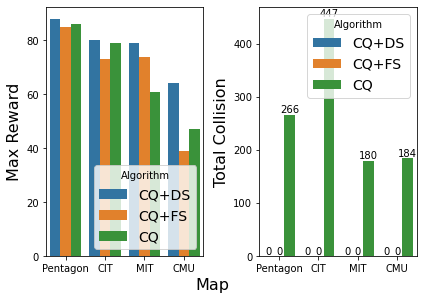}
    \caption{Left-hand side figure is the max reward that agents obtained during learning. The right-hand side figure is the total collision between agents during learning.}
    \label{fig:exp-grid-rew-col}
\end{subfigure} \\
\begin{subfigure}[t]{\linewidth}
    \centering
    \includegraphics[width=\textwidth]{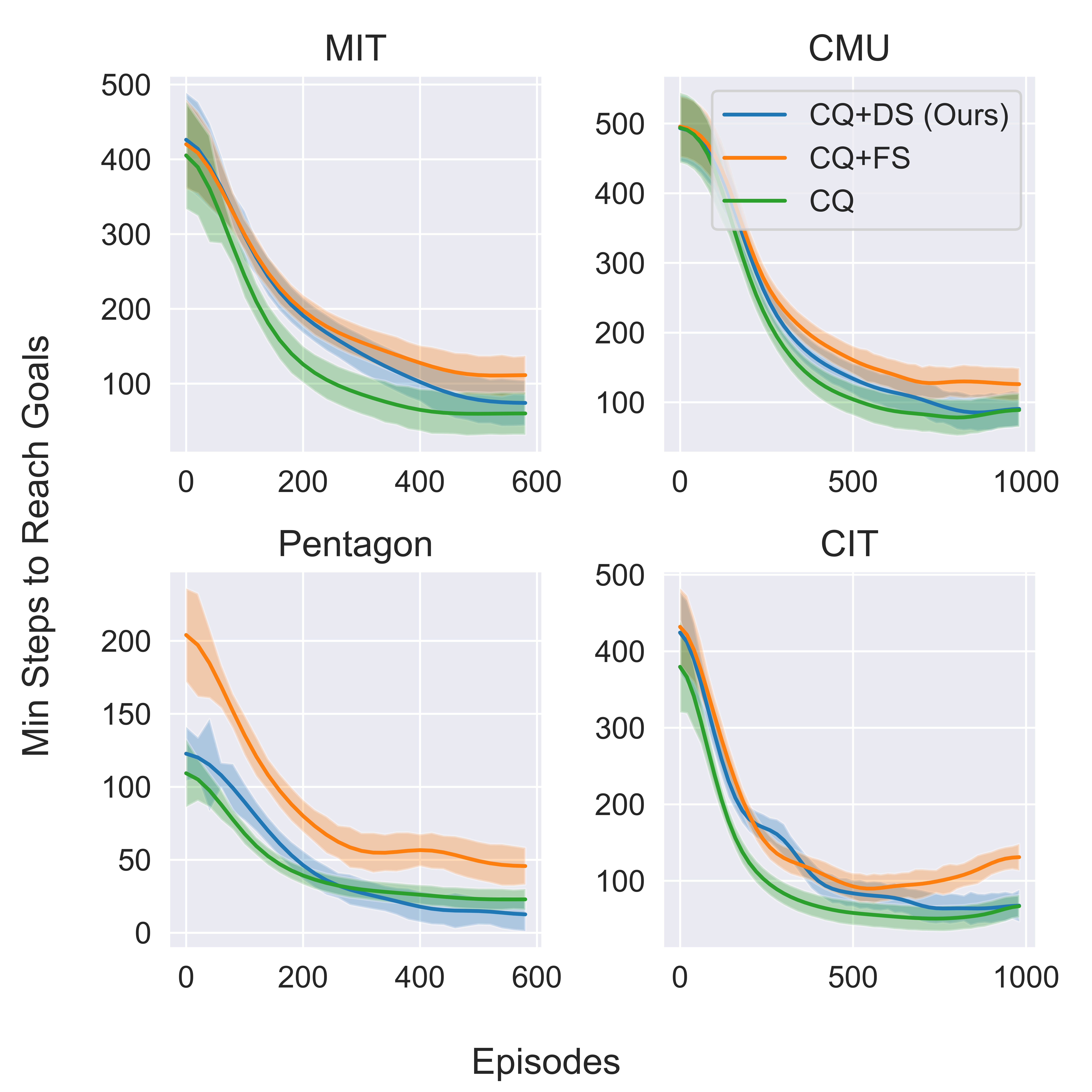}
    \caption{The \textit{Min step to reach goals} value is an indicator of the optimality of learned policies.}
    \label{fig:exp-grid-step}
\end{subfigure}

\caption{Gridworld experiment. \textit{CQ+DS}, \textit{CQ+FS}, and \textit{CQ} denote CQ-Learning with dynamic shielding, CQ-Learning with factored shielding, and CQ-Learning without shielding.}
    \label{fig:exp-grid}
\end{figure}

\textbf{Experiment Setup}. Figure~\ref{fig:exp_grid} shows six maps of grid world benchmark environments adapted from the \textit{gridworld}~\cite{gridworld}. Each map has four agents learning to navigate while avoiding obstacles in the environment. The action set is $\mathcal{A} = \{stay, up, down, left, right\}.$ We randomly assign a unique target to each agent. Once an agent reaches its target, it stays there until all agents reach their goals. We set sparse rewards for this task, namely, giving a $-1$ living penalty, $-10$ collision penalty, and $+100$ for reaching the target.

\begin{table*}
\centering
\begin{tabular}{cc cc cc cc cc cc} 
\toprule
\multicolumn{2}{c}{}          & \multicolumn{2}{c}{MADDPG+DS}         & \multicolumn{2}{c}{MADDPG+Safe}       & \multicolumn{2}{c}{DDPG+DS}           & \multicolumn{2}{c}{DDPG+Safe} & \multicolumn{2}{c}{MADDPG+CBF}    \\ 
\hline
Task                       & N  & REW        & COL             & REW        & COL        & REW        & COL             & REW        & COL             & REW        & COL         \\ 
\hline
\multirow{3}{*}{Spread}    & 4  & $-77 \pm 8$   & \textbf{0.0} & $-84 \pm 6 $  & $1.3 \pm 0.9$   & $-76 \pm 8$   & \textbf{0.0} & $-82 \pm 6 $  & $0.9 \pm 1.0$ & $-75 \pm 6$ & \textbf{0.0}   \\
                           & 8  & $-112 \pm 10$ & \textbf{0.0} & $-119 \pm 9 $ & $10.0 \pm 2.2$   & $-113 \pm 10$ & \textbf{0.0} & $-125 \pm 13$ & $7.7 \pm 1.9$ & $-118 \pm 12$ & \textbf{0.0}    \\
                           & 12 & $-129 \pm 9$  & \textbf{0.0} & $-141 \pm 9$  & $38.2 \pm 6.6$  & $-129 \pm 11$ & \textbf{0.0} & $-151 \pm 14$ & $26.5 \pm 6.5$ & $-138 \pm 11$ & \textbf{0.0}  \\ 
\hline
\multirow{3}{*}{Adversary} & 4  & $-25 \pm 3$   & \textbf{0.0} & $-30 \pm 4$   & $1.1 \pm 0.8$   & $-26 \pm 3$   & \textbf{0.0} & $-26 \pm 3$   & $0.9 \pm 1.3$ & $-23 \pm 3$ & \textbf{0.0}    \\
                           & 8  & $-1 \pm 3 $   & \textbf{0.0} & $-5 \pm 8$    & $6.8 \pm 2.0$   & $ -2 \pm 4$   & \textbf{0.0} & $-11 \pm 16$  & $5.5 \pm 1.7$ & $-2 \pm 5$ & \textbf{0.0}    \\
                           & 12 & $21 \pm 9$    & \textbf{0.0} & $16 \pm 8$    & $43.5 \pm 10.6$ & $23 \pm 9$    & \textbf{0.0} & $13 \pm 11$   & $33.2 \pm 10.9$ & $18 \pm 8$ & \textbf{0.0}  \\
\bottomrule
\end{tabular}
\caption{Results comparing the average rewards and collisions of algorithms during the testing phase. In the table, $a\pm b$ denotes the mean and variance of results from 10 independent testing runs. (\textit{REW, COL} denote cumulative rewards and collisions).}
\label{table:mpe-rew-col}

\end{table*}

Figure~\ref{fig:exp_mpe} shows two tasks from the modified MPE~\cite{mpe}, say \textit{simple spread} and \textit{simple adversary}. The goal of \textit{simple spread} task is for agents to cooperate and reach their target while avoiding collisions. The goal of \textit{simple adversary} task is for good agents to navigate to the target and trick the adversary, and the adversary agents try to reach the target while avoiding collisions. These tasks are more difficult than the \textit{gridworld} in two aspects:
\begin{enumerate}
    \item The state space of MPE is continuous and unbounded.
    \item Agents have more complicated dynamics in MPE, such as momentum and acceleration.
\end{enumerate}
The action set is $\mathcal{A}^{'} = \{stay, up, down, left, right, \textbf{brake}\},$ from which the action controls acceleration. For example, if an agent takes $stay$, it moves at its original velocity instead of staying. We use $brake$ to simulate braking in the real world, where the agent exerts a large deceleration in the direction of velocity until it stops. The $brake$ action obeys the law of kinematics; for example, an agent moving at a higher speed needs a longer distance to brake down. Each agent receives a reward that is inversely proportional to the distance with its target and penalties for collisions.

\begin{figure}[htp]
    \centering
    \includegraphics[width=\linewidth]{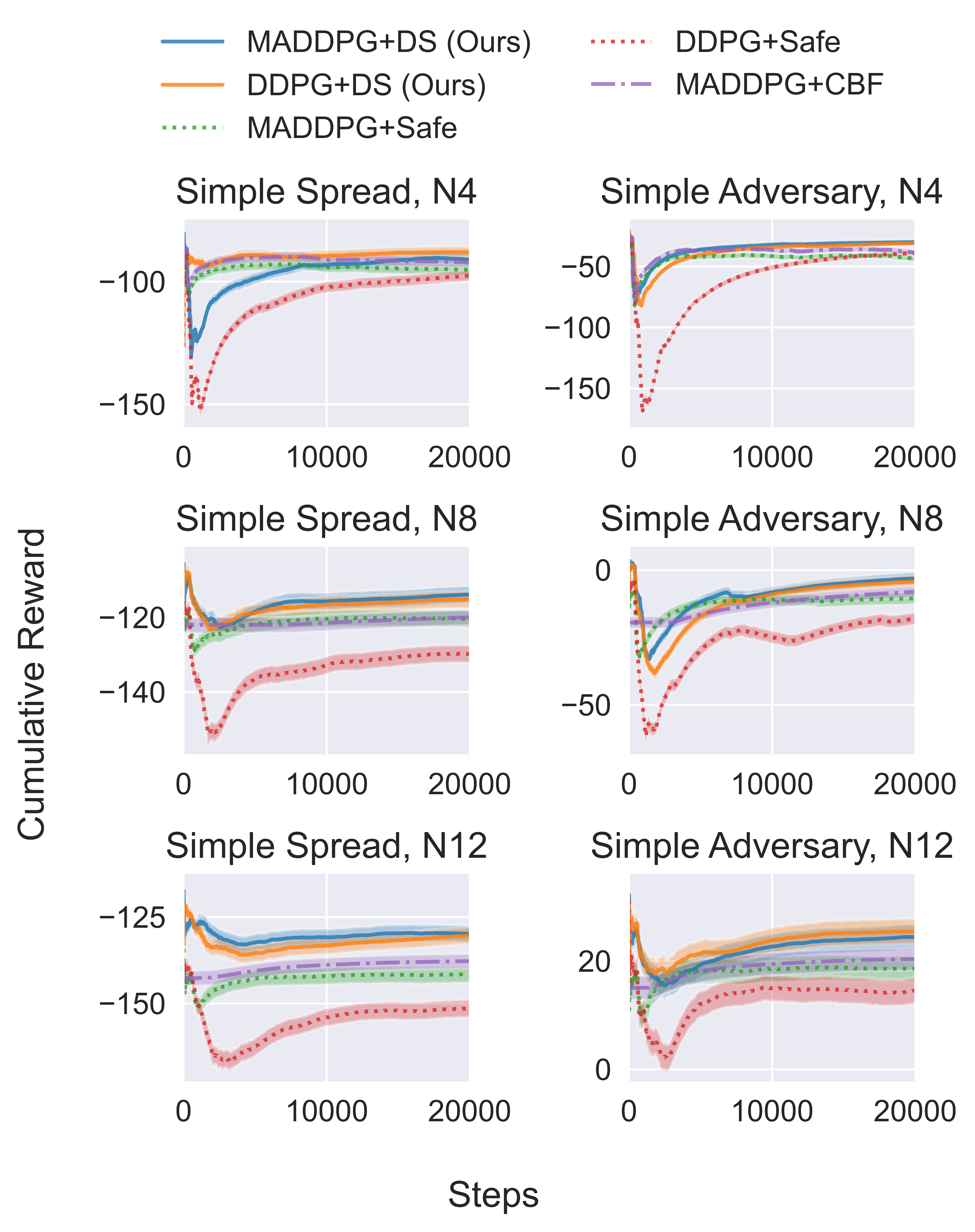}
\caption{Compare Dynamic Shielding (MADDPG/DDPG+DS) with other baselines (MADDPG/DDPG+Safe, MADDPG+CBF).}
    \label{fig:exp-mpe-rew}
\end{figure}

\textbf{Conservative Behavior Evaluation}. We integrate CQ-learning with factored shielding and proposed dynamic shielding. We apply them to four different \textit{gridworld} environments (Figure~\ref{fig:exp_grid}). We evaluate algorithms using maximum rewards, collision counts, and episode steps during the training phase. Results in Figure~\ref{fig:exp-grid-rew-col} show that both factored shielding and dynamic shielding can guarantee collision-free learning in all maps. 
However, dynamic shielding obtains better policies with higher rewards compared to factored shielding and learning without shielding. 
Figure~\ref{fig:exp-grid-step} shows CQ+DS agents need fewer steps to reach the target than CQ+FS. The dynamic shielding policy eventually has comparable performance as CQ-learning without intervention. Therefore, that demonstrates the proposed dynamic shielding mitigates coordination overheads caused by factored shielding with the guarantee of safety.

%IGNORE This table, because it has a minor impact.
\iffalse

\begin{table}
\centering
\begin{tabular}{|c|c|c|c|c|} 
\hline
\multicolumn{5}{|c|}{Task: Simple Spread}                                         \\ 
\hline
N  & \textbf{MADDPG+DS} & MADDPG           & \textbf{DDPG+DS} & DDPG              \\
4  & $1.47 \pm 0.30$    & $1.61\pm 0.29 $  & $1.46 \pm 0.29 $ & $1.54 \pm 0.23 $  \\
8  & $2.57 \pm 0.29$    & $2.59 \pm 0.30$  & $2.58 \pm 0.31 $ & $2.61 \pm 0.3 $   \\
12 & $3.13 \pm 0.29$    & $3.14 \pm 0.29 $ & $3.13 \pm 0.30 $ & $3.17 \pm 0.28 $  \\ 
\hline
\multicolumn{5}{|c|}{Task: Simple Adversary}                                      \\ 
\hline
N  & \textbf{MADDPG+DS} & MADDPG           & \textbf{DDPG+DS} & DDPG              \\
4  & $0.28 \pm 0.12$    & $0.28 \pm 0.12 $ & $0.27 \pm 0.12 $ & $0.26 \pm 0.17 $  \\
8  & $0.40 \pm 0.14$    & $0.39 \pm 0.15$  & $0.39 \pm 0.16$  & $0.41 \pm 0.14$   \\
12 & $1.98 \pm 0.33$    & $1.97 \pm 0.32$  & $1.99 \pm 0.32$  & $1.98 \pm 0.3$    \\
\hline
\end{tabular}
\caption{Results comparing the min distance (minimum distance that agents can reach the targets) of algorithms during testing phase.}
\label{table:min-dist}

\end{table}

\fi

\textbf{Scalability Evaluation}. We evaluate the performance of the dynamic shielding when the state space and the number of agents scale up. We integrate DDPG and MADDPG with dynamic shielding and apply them to the modified MPE environment (shown in Figure~\ref{fig:exp_mpe}). The state space of MPE is a scale-up of gridworld as we described previously. Factored shielding fails in this unbounded environment since we cannot synthesize shields for the entire state space beforehand. 
%We incorporate safety rewards into the reward function and denote this safe version of DDPG/MADDPG as DDPG/MADDPG+Safe. We also implement MADDPG+CBF, which leverages the control barrier function to ensure safety, and we follow the setting in~\cite{cai2021safe}.
We consider two baseline algorithms that incorporate safety mechanisms into DDPG/MADDPG: DDPG/MADDPG+Safe, which adds safety rewards to the reward function, and MADDPG+CBF, which enforces safety using control barrier functions. For MADDPG+CBF, we follow the barrier functions proposed in~\cite{cai2021safe}.
Table~\ref{table:mpe-rew-col} shows MADDPG/DDPG+DS and MADDPG+CBF guarantee collision-free regardless of the tasks. Whereas, MADDPG/DDPG+Safe constantly have collisions, which increases as the number of agents scales up. Table~\ref{table:mpe-rew-col} also depicts the cumulative rewards during the testing phase. At convergence, MADDPG/DDPG+DS obtains higher rewards than MADDPG/DDPG+Safe. 
The learning curves in Figure~\ref{fig:exp-mpe-rew} demonstrate that the MARL algorithms with dynamic shielding converge faster at higher rewards. 
We also evaluate the performance of dynamic shielding when the number of agents scales up. Table~\ref{table:mpe-rew-col} shows that although the collision of MADDPG/DDPG+Safe increase as agents scales up, MADDPG+DS and MADDPG+CBF always ensures safety and has higher rewards in both environments. 
As the number of agents increases, we observe a widening performance gap between MADDPG+DS and MADDPG+CBF. This suggests that our dynamic shielding still maintains minimal intervention (less conservative) as the number of agents scales up.

% \vspace{-2em}
\begin{figure}[htp]
    \centering
    \includegraphics[width=0.9\linewidth]{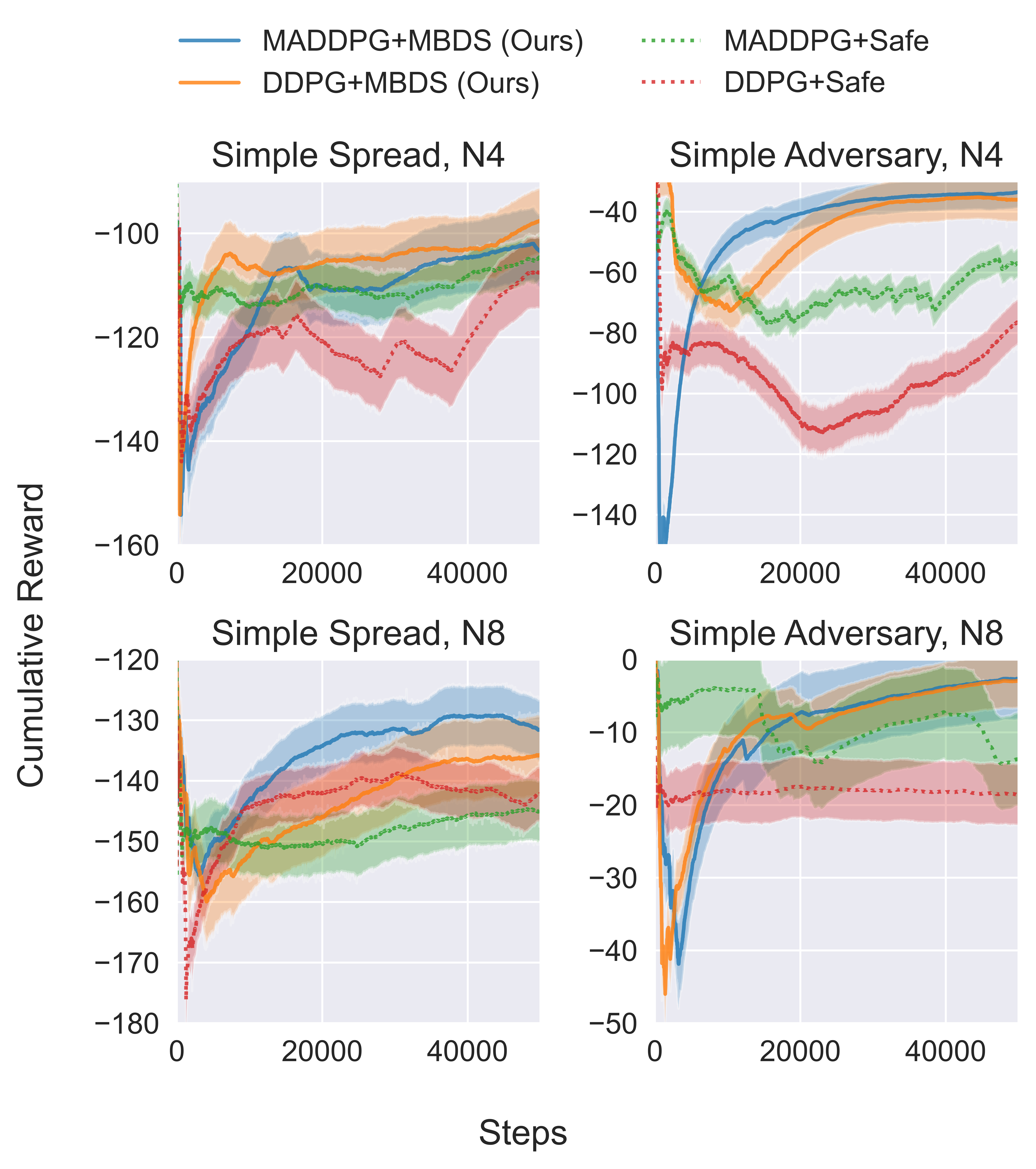}
\caption{Compare Model-Based Dynamic Shielding (MBDS) with other baselines.}
    \label{fig:exp-mb-rew}
\end{figure}

\textbf{Model Based Dynamic Shielding.} In this experiment, we remove the $brake$ action in the environment and evaluate algorithms in the standard MPE \textit{Simple spread} and \textit{Simple Adversary} environments. 
Agents collect $3e5$ roll-outs from the environment to train the world model via Algorithm~\ref{algo:world_model} to learn to predict next step location $x_{t+1}$ based on current location, velocity, and action $\left[x_t, v_t, a_t\right]^T$. Since the temporal information is irrelevant to the dynamics in this environment, we use a $ 32 \times 64 \times 32$ MLP network as the dynamics model. We name this procedure Model-Based Dynamic Shielding (MBDS). We calculate testing phase safety rate by $$r_{safety} =\frac{\sum_i \mathbbm{1}{ (\text{collisions in step i > 0}} )}{\text{number of steps}}.$$ Figure~\ref{fig:exp-mb-rew} demonstrates that MARL with Model-Based Dynamic Shielding (MADDPG+MBDS, DDPG+MBDS) obtains at least not lower cumulative rewards than other baselines (MADDPG+Safe, DDPG+Safe). In addition, when the number of agents increases from 4 to 8, the safety rates of shielding (shown in Figure~\ref{fig:exp-mb-safe}) decrease no more than $5\%$ and keep above $90\%$. However, other baselines has safety rates decrease $10\%$ in \textit{simple spread} and $20\%$ in \textit{simple adversary}. Hence, Model-Based Dynamic Shielding is an effective method to provide safety guarantees for MARL.

\begin{figure}[h]
\centering

\begin{subfigure}[t]{0.7\linewidth}
    \centering
    \includegraphics[width=\textwidth]{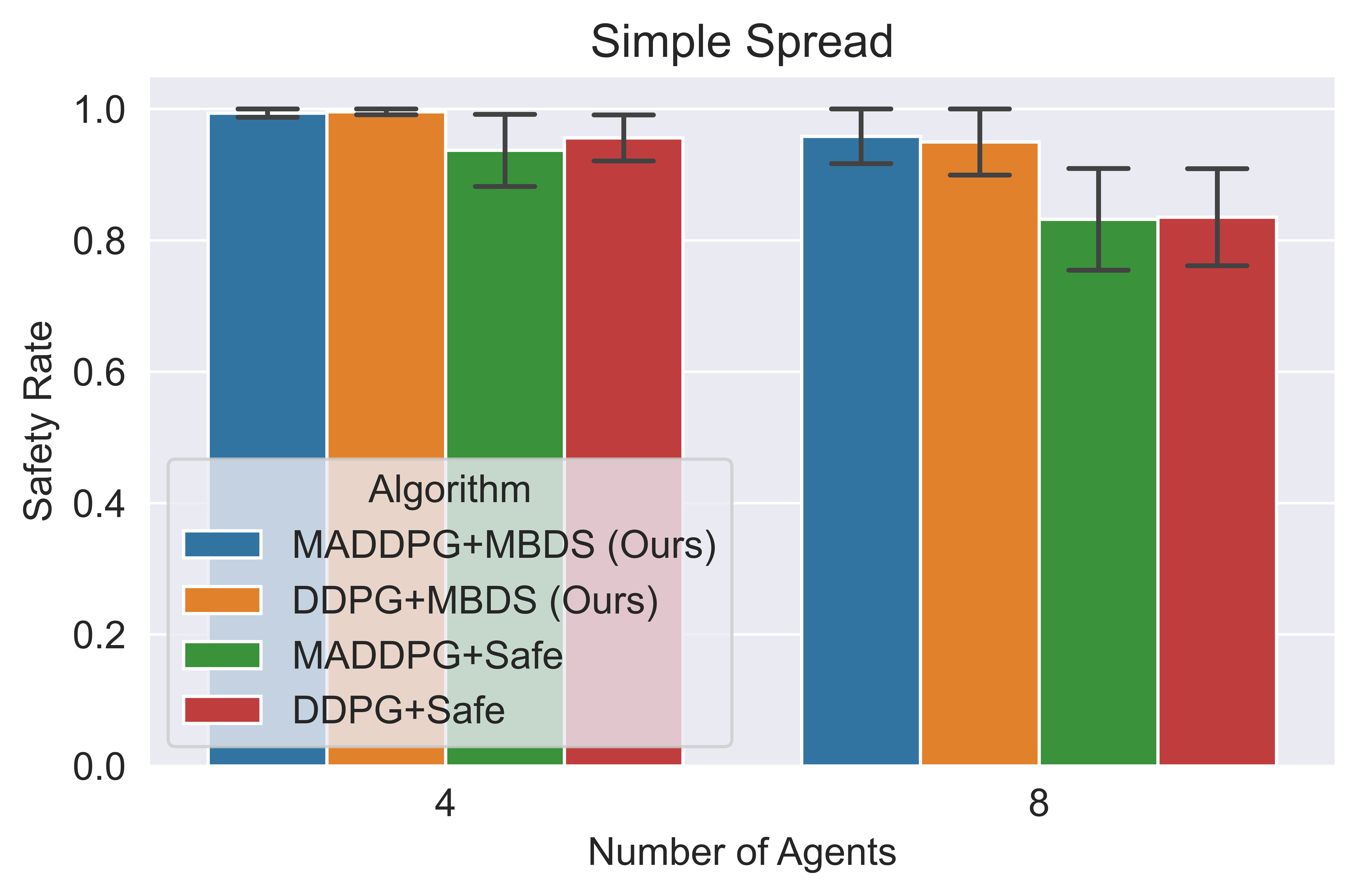}
    % \caption{Episode reward}
    % \label{fig:result_mpe_reward}
\end{subfigure}
\\ 
\begin{subfigure}[t]{0.7\columnwidth}
    \centering
    \includegraphics[width=\textwidth]{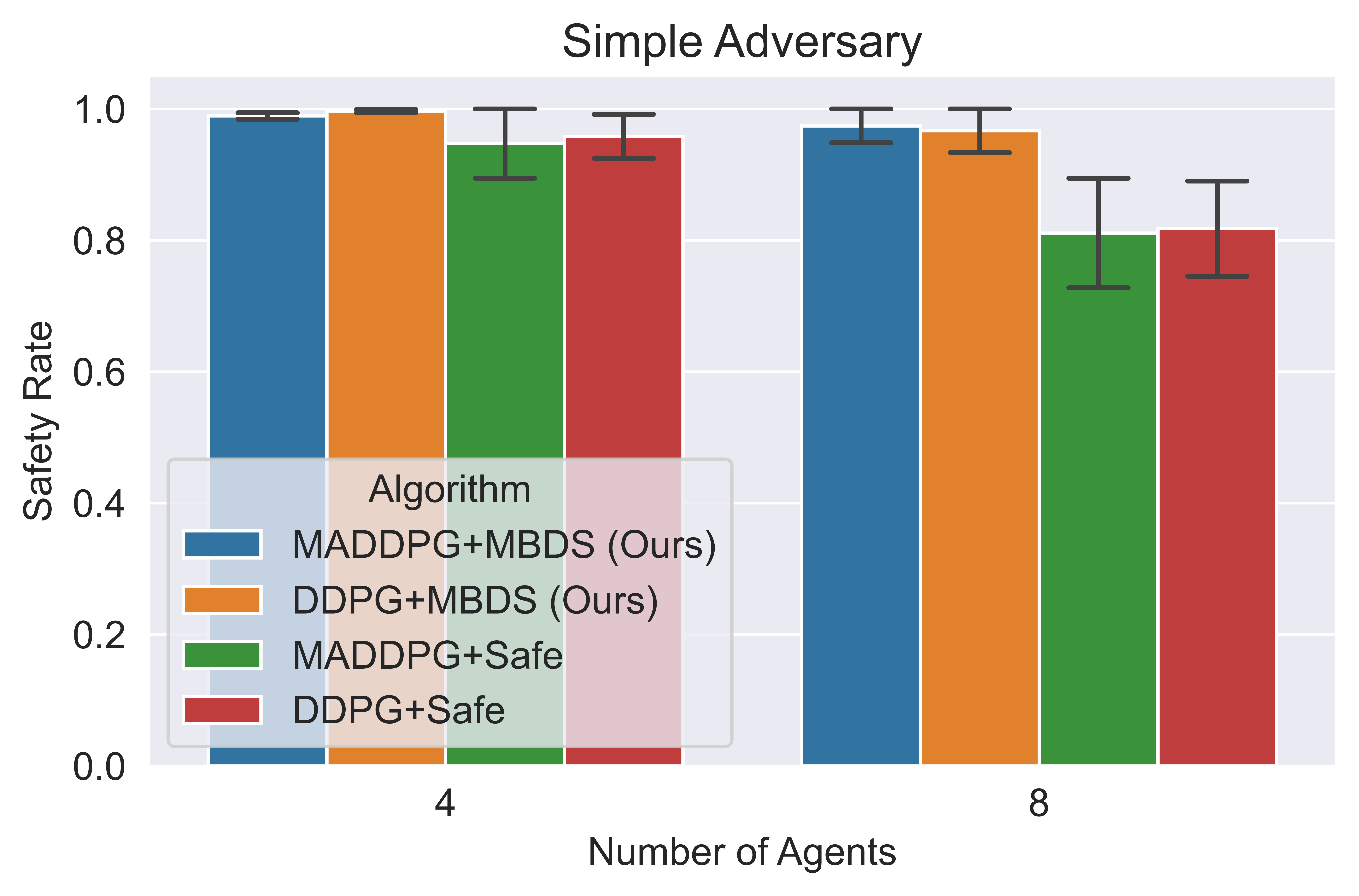}
    
    % \label{fig:result_mpe_collision}
\end{subfigure}

\caption{Safety rate in the MPE tasks.}
\label{fig:exp-mb-safe}
\end{figure}

\iffalse
\textbf{Shield Synthesis Time Cost}. The time complexity of our algorithm is highly depend on the implementation of the reactive game solver, we analyze the efficiency of our shield synthesis method and compare it with the baselines in this work.
\fi

\section{CONCLUSIONS}
This paper presents a novel approach to addressing safe MARL through model-based dynamic shielding. The proposed method minimally interferes with the MARL framework to ensure the safety specification defined by LTL expressions. We also propose an effective technique to synthesize shields in real time and provide theoretical proof of a safety guarantee. In addition, we conduct extensive experiments to demonstrate our algorithm is better than other baselines regarding safety and learning performance in benchmark tasks. 
There are some limitations that we acknowledge. First, our approach does not address the issue of cascading errors that may arise from inaccuracies in the learned dynamics model. This could potentially impair the safety guarantees provided by our framework. Second, we have not analyze the time complexity of our algorithm, which is heavily dependent on the implementation of the reactive game solver.
In our future work, we plan to explore methods for enforcing safety guarantees in the presence of a risk-aware dynamics model, which would help to mitigate the impact of cascading errors. Additionally, we will conduct a thorough analysis of the time complexity of our algorithm to ensure that it can scale to larger and more complex environments.

\clearpage
\begin{acks}
We thank Ingy ElSayed-Aly and Daniel Melcer for the insightful discussion about this work. The work was conducted as part of an undergraduate research experience with the Carnegie Mellon University Robotics Institute Summer Scholars Program. The scholarship funding for Wenli Xiao was provided by the Shenzhen Institute of Artificial Intelligence and Robotics for Society.
\end{acks}

\bibliographystyle{ACM-Reference-Format} 
\bibliography{sample}


\begin{thebibliography}{60}

%%% the Web address.

\ifx \showCODEN    \undefined \def \showCODEN     #1{\unskip}     \fi
\ifx \showDOI      \undefined \def \showDOI       #1{#1}\fi
\ifx \showISBNx    \undefined \def \showISBNx     #1{\unskip}     \fi
\ifx \showISBNxiii \undefined \def \showISBNxiii  #1{\unskip}     \fi
\ifx \showISSN     \undefined \def \showISSN      #1{\unskip}     \fi
\ifx \showLCCN     \undefined \def \showLCCN      #1{\unskip}     \fi
\ifx \shownote     \undefined \def \shownote      #1{#1}          \fi
\ifx \showarticletitle \undefined \def \showarticletitle #1{#1}   \fi
\ifx \showURL      \undefined \def \showURL       {\relax}        \fi
% invisible to TeX
\providecommand\bibfield[2]{#2}
\providecommand\bibinfo[2]{#2}
\providecommand\natexlab[1]{#1}
\providecommand\showeprint[2][]{arXiv:#2}

\bibitem[\protect\citeauthoryear{Alshiekh, Bloem, Ehlers, K{\"o}nighofer,
  Niekum, and Topcu}{Alshiekh et~al\mbox{.}}{2018}]%
        {RL-shielding}
\bibfield{author}{\bibinfo{person}{Mohammed Alshiekh},
  \bibinfo{person}{Roderick Bloem}, \bibinfo{person}{R{\"u}diger Ehlers},
  \bibinfo{person}{Bettina K{\"o}nighofer}, \bibinfo{person}{Scott Niekum},
  {and} \bibinfo{person}{Ufuk Topcu}.} \bibinfo{year}{2018}\natexlab{}.
\newblock \showarticletitle{Safe reinforcement learning via shielding}. In
  \bibinfo{booktitle}{\emph{Proceedings of the AAAI Conference on Artificial
  Intelligence}}, Vol.~\bibinfo{volume}{32}.
\newblock


\bibitem[\protect\citeauthoryear{Alur}{Alur}{2015}]%
        {LTL-app-1}
\bibfield{author}{\bibinfo{person}{Rajeev Alur}.}
  \bibinfo{year}{2015}\natexlab{}.
\newblock \bibinfo{booktitle}{\emph{Principles of cyber-physical systems}}.
\newblock \bibinfo{publisher}{MIT press}.
\newblock


\bibitem[\protect\citeauthoryear{Arel, Liu, Urbanik, and Kohls}{Arel
  et~al\mbox{.}}{2010}]%
        {MARL-Application-transportation-2}
\bibfield{author}{\bibinfo{person}{Itamar Arel}, \bibinfo{person}{Cong Liu},
  \bibinfo{person}{Tom Urbanik}, {and} \bibinfo{person}{Airton~G Kohls}.}
  \bibinfo{year}{2010}\natexlab{}.
\newblock \showarticletitle{Reinforcement learning-based multi-agent system for
  network traffic signal control}.
\newblock \bibinfo{journal}{\emph{IET Intelligent Transport Systems}}
  \bibinfo{volume}{4}, \bibinfo{number}{2} (\bibinfo{year}{2010}),
  \bibinfo{pages}{128--135}.
\newblock


\bibitem[\protect\citeauthoryear{Baier and Katoen}{Baier and Katoen}{2008a}]%
        {LTL-app-2}
\bibfield{author}{\bibinfo{person}{Christel Baier} {and}
  \bibinfo{person}{Joost-Pieter Katoen}.} \bibinfo{year}{2008}\natexlab{a}.
\newblock \bibinfo{booktitle}{\emph{Principles of model checking}}.
\newblock \bibinfo{publisher}{MIT press}.
\newblock


\bibitem[\protect\citeauthoryear{Baier and Katoen}{Baier and Katoen}{2008b}]%
        {LTL-syntax}
\bibfield{author}{\bibinfo{person}{Christel Baier} {and}
  \bibinfo{person}{Joost-Pieter Katoen}.} \bibinfo{year}{2008}\natexlab{b}.
\newblock \bibinfo{booktitle}{\emph{Principles of model checking}}.
\newblock \bibinfo{publisher}{MIT press}.
\newblock


\bibitem[\protect\citeauthoryear{Bastani, Li, and Xu}{Bastani
  et~al\mbox{.}}{2021}]%
        {bastani2021safe}
\bibfield{author}{\bibinfo{person}{Osbert Bastani}, \bibinfo{person}{Shuo Li},
  {and} \bibinfo{person}{Anton Xu}.} \bibinfo{year}{2021}\natexlab{}.
\newblock \showarticletitle{Safe Reinforcement Learning via Statistical Model
  Predictive Shielding.}. In \bibinfo{booktitle}{\emph{Robotics: Science and
  Systems}}. \bibinfo{pages}{1--13}.
\newblock


\bibitem[\protect\citeauthoryear{Bhalla, Ganapathi~Subramanian, and
  Crowley}{Bhalla et~al\mbox{.}}{2020}]%
        {MARL-Application-car-2}
\bibfield{author}{\bibinfo{person}{Sushrut Bhalla}, \bibinfo{person}{Sriram
  Ganapathi~Subramanian}, {and} \bibinfo{person}{Mark Crowley}.}
  \bibinfo{year}{2020}\natexlab{}.
\newblock \showarticletitle{Deep multi agent reinforcement learning for
  autonomous driving}. In \bibinfo{booktitle}{\emph{Canadian Conference on
  Artificial Intelligence}}. Springer, \bibinfo{pages}{67--78}.
\newblock


\bibitem[\protect\citeauthoryear{Borrmann, Wang, Ames, and Egerstedt}{Borrmann
  et~al\mbox{.}}{2015}]%
        {borrmann2015control}
\bibfield{author}{\bibinfo{person}{Urs Borrmann}, \bibinfo{person}{Li Wang},
  \bibinfo{person}{Aaron~D Ames}, {and} \bibinfo{person}{Magnus Egerstedt}.}
  \bibinfo{year}{2015}\natexlab{}.
\newblock \showarticletitle{Control barrier certificates for safe swarm
  behavior}.
\newblock \bibinfo{journal}{\emph{IFAC-PapersOnLine}} \bibinfo{volume}{48},
  \bibinfo{number}{27} (\bibinfo{year}{2015}), \bibinfo{pages}{68--73}.
\newblock


\bibitem[\protect\citeauthoryear{Bozkurt, Wang, Zavlanos, and Pajic}{Bozkurt
  et~al\mbox{.}}{2020}]%
        {bozkurt2020control}
\bibfield{author}{\bibinfo{person}{Alper~Kamil Bozkurt}, \bibinfo{person}{Yu
  Wang}, \bibinfo{person}{Michael~M Zavlanos}, {and} \bibinfo{person}{Miroslav
  Pajic}.} \bibinfo{year}{2020}\natexlab{}.
\newblock \showarticletitle{Control synthesis from linear temporal logic
  specifications using model-free reinforcement learning}. In
  \bibinfo{booktitle}{\emph{2020 IEEE International Conference on Robotics and
  Automation (ICRA)}}. IEEE, \bibinfo{pages}{10349--10355}.
\newblock


\bibitem[\protect\citeauthoryear{Busoniu, Babuska, and De~Schutter}{Busoniu
  et~al\mbox{.}}{2008}]%
        {MARL-survey-1}
\bibfield{author}{\bibinfo{person}{Lucian Busoniu}, \bibinfo{person}{Robert
  Babuska}, {and} \bibinfo{person}{Bart De~Schutter}.}
  \bibinfo{year}{2008}\natexlab{}.
\newblock \showarticletitle{A comprehensive survey of multiagent reinforcement
  learning}.
\newblock \bibinfo{journal}{\emph{IEEE Transactions on Systems, Man, and
  Cybernetics, Part C (Applications and Reviews)}} \bibinfo{volume}{38},
  \bibinfo{number}{2} (\bibinfo{year}{2008}), \bibinfo{pages}{156--172}.
\newblock


\bibitem[\protect\citeauthoryear{Cai, Cao, Lu, Zhang, and Xiong}{Cai
  et~al\mbox{.}}{2021}]%
        {cai2021safe}
\bibfield{author}{\bibinfo{person}{Zhiyuan Cai}, \bibinfo{person}{Huanhui Cao},
  \bibinfo{person}{Wenjie Lu}, \bibinfo{person}{Lin Zhang}, {and}
  \bibinfo{person}{Hao Xiong}.} \bibinfo{year}{2021}\natexlab{}.
\newblock \showarticletitle{Safe multi-agent reinforcement learning through
  decentralized multiple control barrier functions}.
\newblock \bibinfo{journal}{\emph{arXiv preprint arXiv:2103.12553}}
  (\bibinfo{year}{2021}).
\newblock


\bibitem[\protect\citeauthoryear{Chen, Peng, and Grizzle}{Chen
  et~al\mbox{.}}{2017}]%
        {chen2017obstacle}
\bibfield{author}{\bibinfo{person}{Yuxiao Chen}, \bibinfo{person}{Huei Peng},
  {and} \bibinfo{person}{Jessy Grizzle}.} \bibinfo{year}{2017}\natexlab{}.
\newblock \showarticletitle{Obstacle avoidance for low-speed autonomous
  vehicles with barrier function}.
\newblock \bibinfo{journal}{\emph{IEEE Transactions on Control Systems
  Technology}} \bibinfo{volume}{26}, \bibinfo{number}{1}
  (\bibinfo{year}{2017}), \bibinfo{pages}{194--206}.
\newblock


\bibitem[\protect\citeauthoryear{Chen, Singletary, and Ames}{Chen
  et~al\mbox{.}}{2020}]%
        {chen2020guaranteed}
\bibfield{author}{\bibinfo{person}{Yuxiao Chen}, \bibinfo{person}{Andrew
  Singletary}, {and} \bibinfo{person}{Aaron~D Ames}.}
  \bibinfo{year}{2020}\natexlab{}.
\newblock \showarticletitle{Guaranteed obstacle avoidance for multi-robot
  operations with limited actuation: A control barrier function approach}.
\newblock \bibinfo{journal}{\emph{IEEE Control Systems Letters}}
  \bibinfo{volume}{5}, \bibinfo{number}{1} (\bibinfo{year}{2020}),
  \bibinfo{pages}{127--132}.
\newblock


\bibitem[\protect\citeauthoryear{Cheng, Orosz, Murray, and Burdick}{Cheng
  et~al\mbox{.}}{2019}]%
        {cheng2019end}
\bibfield{author}{\bibinfo{person}{Richard Cheng}, \bibinfo{person}{G{\'a}bor
  Orosz}, \bibinfo{person}{Richard~M Murray}, {and} \bibinfo{person}{Joel~W
  Burdick}.} \bibinfo{year}{2019}\natexlab{}.
\newblock \showarticletitle{End-to-end safe reinforcement learning through
  barrier functions for safety-critical continuous control tasks}. In
  \bibinfo{booktitle}{\emph{Proceedings of the AAAI conference on artificial
  intelligence}}, Vol.~\bibinfo{volume}{33}. \bibinfo{pages}{3387--3395}.
\newblock


\bibitem[\protect\citeauthoryear{Chow, Nachum, Duenez-Guzman, and
  Ghavamzadeh}{Chow et~al\mbox{.}}{2018}]%
        {chow2018lyapunov}
\bibfield{author}{\bibinfo{person}{Yinlam Chow}, \bibinfo{person}{Ofir Nachum},
  \bibinfo{person}{Edgar Duenez-Guzman}, {and} \bibinfo{person}{Mohammad
  Ghavamzadeh}.} \bibinfo{year}{2018}\natexlab{}.
\newblock \showarticletitle{A lyapunov-based approach to safe reinforcement
  learning}.
\newblock \bibinfo{journal}{\emph{Advances in neural information processing
  systems}}  \bibinfo{volume}{31} (\bibinfo{year}{2018}).
\newblock


\bibitem[\protect\citeauthoryear{De~Hauwere, Vrancx, and Now{\'e}}{De~Hauwere
  et~al\mbox{.}}{2010}]%
        {cqlearning}
\bibfield{author}{\bibinfo{person}{Yann-Micha{\"e}l De~Hauwere},
  \bibinfo{person}{Peter Vrancx}, {and} \bibinfo{person}{Ann Now{\'e}}.}
  \bibinfo{year}{2010}\natexlab{}.
\newblock \showarticletitle{Learning multi-agent state space representations}.
  In \bibinfo{booktitle}{\emph{Proceedings of the 9th International Conference
  on Autonomous Agents and Multiagent Systems: volume 1-Volume 1}}.
  \bibinfo{pages}{715--722}.
\newblock


\bibitem[\protect\citeauthoryear{Ehlers and Raman}{Ehlers and Raman}{2016}]%
        {slugs}
\bibfield{author}{\bibinfo{person}{R{\"u}diger Ehlers} {and}
  \bibinfo{person}{Vasumathi Raman}.} \bibinfo{year}{2016}\natexlab{}.
\newblock \showarticletitle{Slugs: Extensible gr (1) synthesis}. In
  \bibinfo{booktitle}{\emph{International Conference on Computer Aided
  Verification}}. Springer, \bibinfo{pages}{333--339}.
\newblock


\bibitem[\protect\citeauthoryear{ElSayed-Aly, Bharadwaj, Amato, Ehlers, Topcu,
  and Feng}{ElSayed-Aly et~al\mbox{.}}{2021}]%
        {MARL-shielding}
\bibfield{author}{\bibinfo{person}{Ingy ElSayed-Aly}, \bibinfo{person}{Suda
  Bharadwaj}, \bibinfo{person}{Christopher Amato}, \bibinfo{person}{R{\"u}diger
  Ehlers}, \bibinfo{person}{Ufuk Topcu}, {and} \bibinfo{person}{Lu Feng}.}
  \bibinfo{year}{2021}\natexlab{}.
\newblock \showarticletitle{Safe multi-agent reinforcement learning via
  shielding}.
\newblock \bibinfo{journal}{\emph{arXiv preprint arXiv:2101.11196}}
  (\bibinfo{year}{2021}).
\newblock


\bibitem[\protect\citeauthoryear{Garc{\i}a and Fern{\'a}ndez}{Garc{\i}a and
  Fern{\'a}ndez}{2015}]%
        {MARL-safe-survey-1}
\bibfield{author}{\bibinfo{person}{Javier Garc{\i}a} {and}
  \bibinfo{person}{Fernando Fern{\'a}ndez}.} \bibinfo{year}{2015}\natexlab{}.
\newblock \showarticletitle{A comprehensive survey on safe reinforcement
  learning}.
\newblock \bibinfo{journal}{\emph{Journal of Machine Learning Research}}
  \bibinfo{volume}{16}, \bibinfo{number}{1} (\bibinfo{year}{2015}),
  \bibinfo{pages}{1437--1480}.
\newblock


\bibitem[\protect\citeauthoryear{Ha and Schmidhuber}{Ha and
  Schmidhuber}{2018}]%
        {ha2018world}
\bibfield{author}{\bibinfo{person}{David Ha} {and} \bibinfo{person}{J{\"u}rgen
  Schmidhuber}.} \bibinfo{year}{2018}\natexlab{}.
\newblock \showarticletitle{World models}.
\newblock \bibinfo{journal}{\emph{arXiv preprint arXiv:1803.10122}}
  (\bibinfo{year}{2018}).
\newblock


\bibitem[\protect\citeauthoryear{Hafner, Lillicrap, Fischer, Villegas, Ha, Lee,
  and Davidson}{Hafner et~al\mbox{.}}{2019a}]%
        {RSSM}
\bibfield{author}{\bibinfo{person}{Danijar Hafner}, \bibinfo{person}{Timothy
  Lillicrap}, \bibinfo{person}{Ian Fischer}, \bibinfo{person}{Ruben Villegas},
  \bibinfo{person}{David Ha}, \bibinfo{person}{Honglak Lee}, {and}
  \bibinfo{person}{James Davidson}.} \bibinfo{year}{2019}\natexlab{a}.
\newblock \showarticletitle{Learning latent dynamics for planning from pixels}.
  In \bibinfo{booktitle}{\emph{International conference on machine learning}}.
  PMLR, \bibinfo{pages}{2555--2565}.
\newblock


\bibitem[\protect\citeauthoryear{Hafner, Lillicrap, Fischer, Villegas, Ha, Lee,
  and Davidson}{Hafner et~al\mbox{.}}{2019b}]%
        {hafner2019learning}
\bibfield{author}{\bibinfo{person}{Danijar Hafner}, \bibinfo{person}{Timothy
  Lillicrap}, \bibinfo{person}{Ian Fischer}, \bibinfo{person}{Ruben Villegas},
  \bibinfo{person}{David Ha}, \bibinfo{person}{Honglak Lee}, {and}
  \bibinfo{person}{James Davidson}.} \bibinfo{year}{2019}\natexlab{b}.
\newblock \showarticletitle{Learning latent dynamics for planning from pixels}.
  In \bibinfo{booktitle}{\emph{International conference on machine learning}}.
  PMLR, \bibinfo{pages}{2555--2565}.
\newblock


\bibitem[\protect\citeauthoryear{Hahn, Perez, Schewe, Somenzi, Trivedi, and
  Wojtczak}{Hahn et~al\mbox{.}}{2019}]%
        {hahn2019omega}
\bibfield{author}{\bibinfo{person}{Ernst~Moritz Hahn}, \bibinfo{person}{Mateo
  Perez}, \bibinfo{person}{Sven Schewe}, \bibinfo{person}{Fabio Somenzi},
  \bibinfo{person}{Ashutosh Trivedi}, {and} \bibinfo{person}{Dominik
  Wojtczak}.} \bibinfo{year}{2019}\natexlab{}.
\newblock \showarticletitle{Omega-regular objectives in model-free
  reinforcement learning}. In \bibinfo{booktitle}{\emph{International
  Conference on Tools and Algorithms for the Construction and Analysis of
  Systems}}. Springer, \bibinfo{pages}{395--412}.
\newblock


\bibitem[\protect\citeauthoryear{Hasanbeig, Abate, and Kroening}{Hasanbeig
  et~al\mbox{.}}{2020a}]%
        {RL-LTL-1}
\bibfield{author}{\bibinfo{person}{Mohammadhosein Hasanbeig},
  \bibinfo{person}{Alessandro Abate}, {and} \bibinfo{person}{Daniel Kroening}.}
  \bibinfo{year}{2020}\natexlab{a}.
\newblock \showarticletitle{Cautious reinforcement learning with logical
  constraints}.
\newblock \bibinfo{journal}{\emph{arXiv preprint arXiv:2002.12156}}
  (\bibinfo{year}{2020}).
\newblock


\bibitem[\protect\citeauthoryear{Hasanbeig, Abate, and Kroening}{Hasanbeig
  et~al\mbox{.}}{2020b}]%
        {hasanbeig2020cautious}
\bibfield{author}{\bibinfo{person}{Mohammadhosein Hasanbeig},
  \bibinfo{person}{Alessandro Abate}, {and} \bibinfo{person}{Daniel Kroening}.}
  \bibinfo{year}{2020}\natexlab{b}.
\newblock \showarticletitle{Cautious reinforcement learning with logical
  constraints}.
\newblock \bibinfo{journal}{\emph{arXiv preprint arXiv:2002.12156}}
  (\bibinfo{year}{2020}).
\newblock


\bibitem[\protect\citeauthoryear{Jansen, K{\"o}nighofer, Junges, Serban, and
  Bloem}{Jansen et~al\mbox{.}}{2020}]%
        {jansen2020safe}
\bibfield{author}{\bibinfo{person}{Nils Jansen}, \bibinfo{person}{Bettina
  K{\"o}nighofer}, \bibinfo{person}{JSL Junges}, \bibinfo{person}{AC Serban},
  {and} \bibinfo{person}{Roderick Bloem}.} \bibinfo{year}{2020}\natexlab{}.
\newblock \showarticletitle{Safe reinforcement learning using probabilistic
  shields}.
\newblock  (\bibinfo{year}{2020}).
\newblock


\bibitem[\protect\citeauthoryear{K{\"o}nighofer, Alshiekh, Bloem, Humphrey,
  K{\"o}nighofer, Topcu, and Wang}{K{\"o}nighofer et~al\mbox{.}}{2017}]%
        {konighofer2017shield}
\bibfield{author}{\bibinfo{person}{Bettina K{\"o}nighofer},
  \bibinfo{person}{Mohammed Alshiekh}, \bibinfo{person}{Roderick Bloem},
  \bibinfo{person}{Laura Humphrey}, \bibinfo{person}{Robert K{\"o}nighofer},
  \bibinfo{person}{Ufuk Topcu}, {and} \bibinfo{person}{Chao Wang}.}
  \bibinfo{year}{2017}\natexlab{}.
\newblock \showarticletitle{Shield synthesis}.
\newblock \bibinfo{journal}{\emph{Formal Methods in System Design}}
  \bibinfo{volume}{51}, \bibinfo{number}{2} (\bibinfo{year}{2017}),
  \bibinfo{pages}{332--361}.
\newblock


\bibitem[\protect\citeauthoryear{Kress-Gazit, Fainekos, and Pappas}{Kress-Gazit
  et~al\mbox{.}}{2009}]%
        {kress2009temporal}
\bibfield{author}{\bibinfo{person}{Hadas Kress-Gazit},
  \bibinfo{person}{Georgios~E Fainekos}, {and} \bibinfo{person}{George~J
  Pappas}.} \bibinfo{year}{2009}\natexlab{}.
\newblock \showarticletitle{Temporal-logic-based reactive mission and motion
  planning}.
\newblock \bibinfo{journal}{\emph{IEEE transactions on robotics}}
  \bibinfo{volume}{25}, \bibinfo{number}{6} (\bibinfo{year}{2009}),
  \bibinfo{pages}{1370--1381}.
\newblock


\bibitem[\protect\citeauthoryear{Kupferman and Vardi}{Kupferman and
  Vardi}{2001}]%
        {dfa}
\bibfield{author}{\bibinfo{person}{Orna Kupferman} {and}
  \bibinfo{person}{Moshe~Y Vardi}.} \bibinfo{year}{2001}\natexlab{}.
\newblock \showarticletitle{Model checking of safety properties}.
\newblock \bibinfo{journal}{\emph{Formal methods in system design}}
  \bibinfo{volume}{19}, \bibinfo{number}{3} (\bibinfo{year}{2001}),
  \bibinfo{pages}{291--314}.
\newblock


\bibitem[\protect\citeauthoryear{Li and Bastani}{Li and Bastani}{2020}]%
        {li2020robust}
\bibfield{author}{\bibinfo{person}{Shuo Li} {and} \bibinfo{person}{Osbert
  Bastani}.} \bibinfo{year}{2020}\natexlab{}.
\newblock \showarticletitle{Robust model predictive shielding for safe
  reinforcement learning with stochastic dynamics}. In
  \bibinfo{booktitle}{\emph{2020 IEEE International Conference on Robotics and
  Automation (ICRA)}}. IEEE, \bibinfo{pages}{7166--7172}.
\newblock


\bibitem[\protect\citeauthoryear{Lillicrap, Hunt, Pritzel, Heess, Erez, Tassa,
  Silver, and Wierstra}{Lillicrap et~al\mbox{.}}{2015}]%
        {ddpg}
\bibfield{author}{\bibinfo{person}{Timothy~P Lillicrap},
  \bibinfo{person}{Jonathan~J Hunt}, \bibinfo{person}{Alexander Pritzel},
  \bibinfo{person}{Nicolas Heess}, \bibinfo{person}{Tom Erez},
  \bibinfo{person}{Yuval Tassa}, \bibinfo{person}{David Silver}, {and}
  \bibinfo{person}{Daan Wierstra}.} \bibinfo{year}{2015}\natexlab{}.
\newblock \showarticletitle{Continuous control with deep reinforcement
  learning}.
\newblock \bibinfo{journal}{\emph{arXiv preprint arXiv:1509.02971}}
  (\bibinfo{year}{2015}).
\newblock


\bibitem[\protect\citeauthoryear{Lowe, Wu, Tamar, Harb, Abbeel, and
  Mordatch}{Lowe et~al\mbox{.}}{2017}]%
        {mpe}
\bibfield{author}{\bibinfo{person}{Ryan Lowe}, \bibinfo{person}{Yi Wu},
  \bibinfo{person}{Aviv Tamar}, \bibinfo{person}{Jean Harb},
  \bibinfo{person}{Pieter Abbeel}, {and} \bibinfo{person}{Igor Mordatch}.}
  \bibinfo{year}{2017}\natexlab{}.
\newblock \showarticletitle{Multi-Agent Actor-Critic for Mixed
  Cooperative-Competitive Environments}.
\newblock \bibinfo{journal}{\emph{Neural Information Processing Systems
  (NIPS)}} (\bibinfo{year}{2017}).
\newblock


\bibitem[\protect\citeauthoryear{Lu, Zhang, Chen, Ba{\c{s}}ar, and Horesh}{Lu
  et~al\mbox{.}}{2021}]%
        {lu2021decentralized}
\bibfield{author}{\bibinfo{person}{Songtao Lu}, \bibinfo{person}{Kaiqing
  Zhang}, \bibinfo{person}{Tianyi Chen}, \bibinfo{person}{Tamer Ba{\c{s}}ar},
  {and} \bibinfo{person}{Lior Horesh}.} \bibinfo{year}{2021}\natexlab{}.
\newblock \showarticletitle{Decentralized policy gradient descent ascent for
  safe multi-agent reinforcement learning}. In
  \bibinfo{booktitle}{\emph{Proceedings of the AAAI Conference on Artificial
  Intelligence}}, Vol.~\bibinfo{volume}{35}. \bibinfo{pages}{8767--8775}.
\newblock


\bibitem[\protect\citeauthoryear{Luo, Sun, and Kapoor}{Luo
  et~al\mbox{.}}{2020}]%
        {luo2020multi}
\bibfield{author}{\bibinfo{person}{Wenhao Luo}, \bibinfo{person}{Wen Sun},
  {and} \bibinfo{person}{Ashish Kapoor}.} \bibinfo{year}{2020}\natexlab{}.
\newblock \showarticletitle{Multi-robot collision avoidance under uncertainty
  with probabilistic safety barrier certificates}.
\newblock \bibinfo{journal}{\emph{Advances in Neural Information Processing
  Systems}}  \bibinfo{volume}{33} (\bibinfo{year}{2020}),
  \bibinfo{pages}{372--383}.
\newblock


\bibitem[\protect\citeauthoryear{Lyu, Luo, and Dolan}{Lyu
  et~al\mbox{.}}{2021}]%
        {lyu2021probabilistic}
\bibfield{author}{\bibinfo{person}{Yiwei Lyu}, \bibinfo{person}{Wenhao Luo},
  {and} \bibinfo{person}{John~M Dolan}.} \bibinfo{year}{2021}\natexlab{}.
\newblock \showarticletitle{Probabilistic safety-assured adaptive merging
  control for autonomous vehicles}. In \bibinfo{booktitle}{\emph{2021 IEEE
  International Conference on Robotics and Automation (ICRA)}}. IEEE,
  \bibinfo{pages}{10764--10770}.
\newblock


\bibitem[\protect\citeauthoryear{Lyu, Luo, and Dolan}{Lyu
  et~al\mbox{.}}{2022a}]%
        {lyu2022adaptive}
\bibfield{author}{\bibinfo{person}{Yiwei Lyu}, \bibinfo{person}{Wenhao Luo},
  {and} \bibinfo{person}{John~M Dolan}.} \bibinfo{year}{2022}\natexlab{a}.
\newblock \showarticletitle{Adaptive safe merging control for heterogeneous
  autonomous vehicles using parametric control barrier functions}. In
  \bibinfo{booktitle}{\emph{2022 IEEE Intelligent Vehicles Symposium (IV)}}.
  IEEE, \bibinfo{pages}{542--547}.
\newblock


\bibitem[\protect\citeauthoryear{Lyu, Luo, and Dolan}{Lyu
  et~al\mbox{.}}{2022b}]%
        {lyu2022responsibility}
\bibfield{author}{\bibinfo{person}{Yiwei Lyu}, \bibinfo{person}{Wenhao Luo},
  {and} \bibinfo{person}{John~M Dolan}.} \bibinfo{year}{2022}\natexlab{b}.
\newblock \showarticletitle{Responsibility-associated Multi-agent Collision
  Avoidance with Social Preferences}. In \bibinfo{booktitle}{\emph{2022 IEEE
  25th International Conference on Intelligent Transportation Systems (ITSC)}}.
  IEEE, \bibinfo{pages}{3645--3651}.
\newblock


\bibitem[\protect\citeauthoryear{Melo and Veloso}{Melo and Veloso}{2009a}]%
        {MARL-env-1}
\bibfield{author}{\bibinfo{person}{Francisco~S Melo} {and}
  \bibinfo{person}{Manuela Veloso}.} \bibinfo{year}{2009}\natexlab{a}.
\newblock \showarticletitle{Learning of coordination: Exploiting sparse
  interactions in multiagent systems}. In \bibinfo{booktitle}{\emph{Proceedings
  of The 8th International Conference on Autonomous Agents and Multiagent
  Systems-Volume 2}}. \bibinfo{pages}{773--780}.
\newblock


\bibitem[\protect\citeauthoryear{Melo and Veloso}{Melo and Veloso}{2009b}]%
        {gridworld}
\bibfield{author}{\bibinfo{person}{Francisco~S Melo} {and}
  \bibinfo{person}{Manuela Veloso}.} \bibinfo{year}{2009}\natexlab{b}.
\newblock \showarticletitle{Learning of coordination: Exploiting sparse
  interactions in multiagent systems}. In \bibinfo{booktitle}{\emph{Proceedings
  of The 8th International Conference on Autonomous Agents and Multiagent
  Systems-Volume 2}}. Citeseer, \bibinfo{pages}{773--780}.
\newblock


\bibitem[\protect\citeauthoryear{Okdinawati, Simatupang, and
  Sunitiyoso}{Okdinawati et~al\mbox{.}}{2017}]%
        {MARL-Application-transportation-1}
\bibfield{author}{\bibinfo{person}{Liane Okdinawati}, \bibinfo{person}{Togar~M
  Simatupang}, {and} \bibinfo{person}{Yos Sunitiyoso}.}
  \bibinfo{year}{2017}\natexlab{}.
\newblock \showarticletitle{Multi-agent reinforcement learning for value
  co-creation of collaborative transportation management (CTM)}.
\newblock \bibinfo{journal}{\emph{International Journal of Information Systems
  and Supply Chain Management (IJISSCM)}} \bibinfo{volume}{10},
  \bibinfo{number}{3} (\bibinfo{year}{2017}), \bibinfo{pages}{84--95}.
\newblock


\bibitem[\protect\citeauthoryear{Pecka and Svoboda}{Pecka and Svoboda}{2014}]%
        {safe-notation}
\bibfield{author}{\bibinfo{person}{Martin Pecka} {and} \bibinfo{person}{Tomas
  Svoboda}.} \bibinfo{year}{2014}\natexlab{}.
\newblock \showarticletitle{Safe exploration techniques for reinforcement
  learning--an overview}. In \bibinfo{booktitle}{\emph{International Workshop
  on Modelling and Simulation for Autonomous Systems}}. Springer,
  \bibinfo{pages}{357--375}.
\newblock


\bibitem[\protect\citeauthoryear{Perrusqu{\'\i}a, Yu, and Li}{Perrusqu{\'\i}a
  et~al\mbox{.}}{2020}]%
        {MARL-Application-robot-1}
\bibfield{author}{\bibinfo{person}{Adolfo Perrusqu{\'\i}a},
  \bibinfo{person}{Wen Yu}, {and} \bibinfo{person}{Xiaoou Li}.}
  \bibinfo{year}{2020}\natexlab{}.
\newblock \showarticletitle{Redundant robot control using multi agent
  reinforcement learning}. In \bibinfo{booktitle}{\emph{2020 IEEE 16th
  International Conference on Automation Science and Engineering (CASE)}}.
  IEEE, \bibinfo{pages}{1650--1655}.
\newblock


\bibitem[\protect\citeauthoryear{Pnueli}{Pnueli}{1977}]%
        {LTL}
\bibfield{author}{\bibinfo{person}{Amir Pnueli}.}
  \bibinfo{year}{1977}\natexlab{}.
\newblock \showarticletitle{The temporal logic of programs}. In
  \bibinfo{booktitle}{\emph{18th Annual Symposium on Foundations of Computer
  Science (sfcs 1977)}}. ieee, \bibinfo{pages}{46--57}.
\newblock


\bibitem[\protect\citeauthoryear{Prajna, Jadbabaie, and Pappas}{Prajna
  et~al\mbox{.}}{2007}]%
        {prajna2007framework}
\bibfield{author}{\bibinfo{person}{Stephen Prajna}, \bibinfo{person}{Ali
  Jadbabaie}, {and} \bibinfo{person}{George~J Pappas}.}
  \bibinfo{year}{2007}\natexlab{}.
\newblock \showarticletitle{A framework for worst-case and stochastic safety
  verification using barrier certificates}.
\newblock \bibinfo{journal}{\emph{IEEE Trans. Automat. Control}}
  \bibinfo{volume}{52}, \bibinfo{number}{8} (\bibinfo{year}{2007}),
  \bibinfo{pages}{1415--1428}.
\newblock


\bibitem[\protect\citeauthoryear{Qin, Zhang, Chen, Chen, and Fan}{Qin
  et~al\mbox{.}}{2021}]%
        {qin2021learning}
\bibfield{author}{\bibinfo{person}{Zengyi Qin}, \bibinfo{person}{Kaiqing
  Zhang}, \bibinfo{person}{Yuxiao Chen}, \bibinfo{person}{Jingkai Chen}, {and}
  \bibinfo{person}{Chuchu Fan}.} \bibinfo{year}{2021}\natexlab{}.
\newblock \showarticletitle{Learning safe multi-agent control with
  decentralized neural barrier certificates}.
\newblock \bibinfo{journal}{\emph{arXiv preprint arXiv:2101.05436}}
  (\bibinfo{year}{2021}).
\newblock


\bibitem[\protect\citeauthoryear{Rozier}{Rozier}{2011}]%
        {rozier2011linear}
\bibfield{author}{\bibinfo{person}{Kristin~Y Rozier}.}
  \bibinfo{year}{2011}\natexlab{}.
\newblock \showarticletitle{Linear temporal logic symbolic model checking}.
\newblock \bibinfo{journal}{\emph{Computer Science Review}}
  \bibinfo{volume}{5}, \bibinfo{number}{2} (\bibinfo{year}{2011}),
  \bibinfo{pages}{163--203}.
\newblock


\bibitem[\protect\citeauthoryear{Shalev-Shwartz, Shammah, and
  Shashua}{Shalev-Shwartz et~al\mbox{.}}{2016}]%
        {MARL-Application-car-1}
\bibfield{author}{\bibinfo{person}{Shai Shalev-Shwartz},
  \bibinfo{person}{Shaked Shammah}, {and} \bibinfo{person}{Amnon Shashua}.}
  \bibinfo{year}{2016}\natexlab{}.
\newblock \showarticletitle{Safe, multi-agent, reinforcement learning for
  autonomous driving}.
\newblock \bibinfo{journal}{\emph{arXiv preprint arXiv:1610.03295}}
  (\bibinfo{year}{2016}).
\newblock


\bibitem[\protect\citeauthoryear{Srinivasan and Coogan}{Srinivasan and
  Coogan}{2020}]%
        {srinivasan2020control}
\bibfield{author}{\bibinfo{person}{Mohit Srinivasan} {and}
  \bibinfo{person}{Samuel Coogan}.} \bibinfo{year}{2020}\natexlab{}.
\newblock \showarticletitle{Control of mobile robots using barrier functions
  under temporal logic specifications}.
\newblock \bibinfo{journal}{\emph{IEEE Transactions on Robotics}}
  \bibinfo{volume}{37}, \bibinfo{number}{2} (\bibinfo{year}{2020}),
  \bibinfo{pages}{363--374}.
\newblock


\bibitem[\protect\citeauthoryear{Taylor, Singletary, Yue, and Ames}{Taylor
  et~al\mbox{.}}{2020}]%
        {taylor2020learning}
\bibfield{author}{\bibinfo{person}{Andrew Taylor}, \bibinfo{person}{Andrew
  Singletary}, \bibinfo{person}{Yisong Yue}, {and} \bibinfo{person}{Aaron
  Ames}.} \bibinfo{year}{2020}\natexlab{}.
\newblock \showarticletitle{Learning for safety-critical control with control
  barrier functions}. In \bibinfo{booktitle}{\emph{Learning for Dynamics and
  Control}}. PMLR, \bibinfo{pages}{708--717}.
\newblock


\bibitem[\protect\citeauthoryear{Thananjeyan, Balakrishna, Nair, Luo,
  Srinivasan, Hwang, Gonzalez, Ibarz, Finn, and Goldberg}{Thananjeyan
  et~al\mbox{.}}{2021}]%
        {thananjeyan2021recovery}
\bibfield{author}{\bibinfo{person}{Brijen Thananjeyan}, \bibinfo{person}{Ashwin
  Balakrishna}, \bibinfo{person}{Suraj Nair}, \bibinfo{person}{Michael Luo},
  \bibinfo{person}{Krishnan Srinivasan}, \bibinfo{person}{Minho Hwang},
  \bibinfo{person}{Joseph~E Gonzalez}, \bibinfo{person}{Julian Ibarz},
  \bibinfo{person}{Chelsea Finn}, {and} \bibinfo{person}{Ken Goldberg}.}
  \bibinfo{year}{2021}\natexlab{}.
\newblock \showarticletitle{Recovery rl: Safe reinforcement learning with
  learned recovery zones}.
\newblock \bibinfo{journal}{\emph{IEEE Robotics and Automation Letters}}
  \bibinfo{volume}{6}, \bibinfo{number}{3} (\bibinfo{year}{2021}),
  \bibinfo{pages}{4915--4922}.
\newblock


\bibitem[\protect\citeauthoryear{Ulusoy, Smith, Ding, Belta, and Rus}{Ulusoy
  et~al\mbox{.}}{2013}]%
        {ulusoy2013optimality}
\bibfield{author}{\bibinfo{person}{Alphan Ulusoy}, \bibinfo{person}{Stephen~L
  Smith}, \bibinfo{person}{Xu~Chu Ding}, \bibinfo{person}{Calin Belta}, {and}
  \bibinfo{person}{Daniela Rus}.} \bibinfo{year}{2013}\natexlab{}.
\newblock \showarticletitle{Optimality and robustness in multi-robot path
  planning with temporal logic constraints}.
\newblock \bibinfo{journal}{\emph{The International Journal of Robotics
  Research}} \bibinfo{volume}{32}, \bibinfo{number}{8} (\bibinfo{year}{2013}),
  \bibinfo{pages}{889--911}.
\newblock


\bibitem[\protect\citeauthoryear{Van~Koevering, Lyu, Luo, and
  Dolan}{Van~Koevering et~al\mbox{.}}{2022}]%
        {van2022provable}
\bibfield{author}{\bibinfo{person}{Spencer Van~Koevering},
  \bibinfo{person}{Yiwei Lyu}, \bibinfo{person}{Wenhao Luo}, {and}
  \bibinfo{person}{John Dolan}.} \bibinfo{year}{2022}\natexlab{}.
\newblock \showarticletitle{Provable Probabilistic Safety and
  Feasibility-Assured Control for Autonomous Vehicles using Exponential Control
  Barrier Functions}. In \bibinfo{booktitle}{\emph{2022 IEEE Intelligent
  Vehicles Symposium (IV)}}. IEEE, \bibinfo{pages}{952--957}.
\newblock


\bibitem[\protect\citeauthoryear{Wachi and Sui}{Wachi and Sui}{2020a}]%
        {wachi2020safe}
\bibfield{author}{\bibinfo{person}{Akifumi Wachi} {and} \bibinfo{person}{Yanan
  Sui}.} \bibinfo{year}{2020}\natexlab{a}.
\newblock \showarticletitle{Safe reinforcement learning in constrained Markov
  decision processes}. In \bibinfo{booktitle}{\emph{International Conference on
  Machine Learning}}. PMLR, \bibinfo{pages}{9797--9806}.
\newblock


\bibitem[\protect\citeauthoryear{Wachi and Sui}{Wachi and Sui}{2020b}]%
        {cmdp-1}
\bibfield{author}{\bibinfo{person}{Akifumi Wachi} {and} \bibinfo{person}{Yanan
  Sui}.} \bibinfo{year}{2020}\natexlab{b}.
\newblock \showarticletitle{Safe reinforcement learning in constrained Markov
  decision processes}. In \bibinfo{booktitle}{\emph{International Conference on
  Machine Learning}}. PMLR, \bibinfo{pages}{9797--9806}.
\newblock


\bibitem[\protect\citeauthoryear{Wang, Ames, and Egerstedt}{Wang
  et~al\mbox{.}}{2017}]%
        {wang2017safety}
\bibfield{author}{\bibinfo{person}{Li Wang}, \bibinfo{person}{Aaron~D Ames},
  {and} \bibinfo{person}{Magnus Egerstedt}.} \bibinfo{year}{2017}\natexlab{}.
\newblock \showarticletitle{Safety barrier certificates for collisions-free
  multirobot systems}.
\newblock \bibinfo{journal}{\emph{IEEE Transactions on Robotics}}
  \bibinfo{volume}{33}, \bibinfo{number}{3} (\bibinfo{year}{2017}),
  \bibinfo{pages}{661--674}.
\newblock


\bibitem[\protect\citeauthoryear{Wieland and Allg{\"o}wer}{Wieland and
  Allg{\"o}wer}{2007}]%
        {wieland2007constructive}
\bibfield{author}{\bibinfo{person}{Peter Wieland} {and} \bibinfo{person}{Frank
  Allg{\"o}wer}.} \bibinfo{year}{2007}\natexlab{}.
\newblock \showarticletitle{Constructive safety using control barrier
  functions}.
\newblock \bibinfo{journal}{\emph{IFAC Proceedings Volumes}}
  \bibinfo{volume}{40}, \bibinfo{number}{12} (\bibinfo{year}{2007}),
  \bibinfo{pages}{462--467}.
\newblock


\bibitem[\protect\citeauthoryear{Yu, Wang, and Feng}{Yu et~al\mbox{.}}{2019}]%
        {MARL-Application-robot-2}
\bibfield{author}{\bibinfo{person}{Chao Yu}, \bibinfo{person}{Xin Wang}, {and}
  \bibinfo{person}{Zhanbo Feng}.} \bibinfo{year}{2019}\natexlab{}.
\newblock \showarticletitle{Coordinated multiagent reinforcement learning for
  teams of mobile sensing robots}. In \bibinfo{booktitle}{\emph{Proceedings of
  the 18th international conference on autonomous agents and multiagent
  systems}}. \bibinfo{pages}{2297--2299}.
\newblock


\bibitem[\protect\citeauthoryear{Zeng, Zhang, and Sreenath}{Zeng
  et~al\mbox{.}}{2021}]%
        {zeng2021safety}
\bibfield{author}{\bibinfo{person}{Jun Zeng}, \bibinfo{person}{Bike Zhang},
  {and} \bibinfo{person}{Koushil Sreenath}.} \bibinfo{year}{2021}\natexlab{}.
\newblock \showarticletitle{Safety-critical model predictive control with
  discrete-time control barrier function}. In \bibinfo{booktitle}{\emph{2021
  American Control Conference (ACC)}}. IEEE, \bibinfo{pages}{3882--3889}.
\newblock


\bibitem[\protect\citeauthoryear{Zhang, Yang, and Ba{\c{s}}ar}{Zhang
  et~al\mbox{.}}{2021}]%
        {MARL-survey-2}
\bibfield{author}{\bibinfo{person}{Kaiqing Zhang}, \bibinfo{person}{Zhuoran
  Yang}, {and} \bibinfo{person}{Tamer Ba{\c{s}}ar}.}
  \bibinfo{year}{2021}\natexlab{}.
\newblock \showarticletitle{Multi-agent reinforcement learning: A selective
  overview of theories and algorithms}.
\newblock \bibinfo{journal}{\emph{Handbook of Reinforcement Learning and
  Control}} (\bibinfo{year}{2021}), \bibinfo{pages}{321--384}.
\newblock


\bibitem[\protect\citeauthoryear{Zhou, Luo, Villella, Yang, Rusu, Miao, Zhang,
  Alban, Fadakar, Chen, et~al\mbox{.}}{Zhou et~al\mbox{.}}{2020}]%
        {MARL-Application-car-3}
\bibfield{author}{\bibinfo{person}{Ming Zhou}, \bibinfo{person}{Jun Luo},
  \bibinfo{person}{Julian Villella}, \bibinfo{person}{Yaodong Yang},
  \bibinfo{person}{David Rusu}, \bibinfo{person}{Jiayu Miao},
  \bibinfo{person}{Weinan Zhang}, \bibinfo{person}{Montgomery Alban},
  \bibinfo{person}{Iman Fadakar}, \bibinfo{person}{Zheng Chen},
  {et~al\mbox{.}}} \bibinfo{year}{2020}\natexlab{}.
\newblock \showarticletitle{Smarts: Scalable multi-agent reinforcement learning
  training school for autonomous driving}.
\newblock \bibinfo{journal}{\emph{arXiv preprint arXiv:2010.09776}}
  (\bibinfo{year}{2020}).
\newblock


\end{thebibliography}

%%%%%%%%%%%%%%%%%%%%%%%%%%%%%%%%%%%%%%%%%%%%%%%%%%%%%%%%%%%%%%%%%%%%%%%%

\end{document}